\documentclass[11pt]{article}
\usepackage[T1]{fontenc}
\usepackage[utf8]{inputenc}
\usepackage[margin=1in]{geometry}
\usepackage{booktabs}
\usepackage{amsmath}
\usepackage{amssymb}
\usepackage{amsthm}
\usepackage{critical-state-preprint}
\theoremstyle{plain}
\newtheorem{theorem}{Theorem}
\newtheorem{lemma}{Lemma}
\newtheorem{proposition}{Proposition}
\newtheorem{corollary}{Corollary}
\theoremstyle{definition}
\newtheorem{definition}{Definition}

\usepackage{enumitem}
\setlist{itemsep=2pt,topsep=3pt,parsep=0pt}
\usepackage[table]{xcolor}
\definecolor{bestrow}{gray}{0.92}
\usepackage{array}
\usepackage{multirow}
\usepackage{siunitx}
\usepackage[font=small,labelfont=bf,labelsep=period]{caption}
\newcommand{\floatnote}[1]{%
  \par\vspace{2pt}%
  {\footnotesize\raggedright\textit{Note.} #1\par}}
\newcommand{\readingguide}[1]{%
  \par\vspace{2pt}%
  {\footnotesize\raggedright\textit{Reading guide.} #1\par}}
\usepackage[pdfusetitle,bookmarksnumbered=true,colorlinks=true,linkcolor=blue!45!black,citecolor=blue!45!black,urlcolor=blue!45!black]{hyperref}
\usepackage{graphicx}
\usepackage{flafter}
\usepackage{float}
\usepackage{placeins}
\usepackage{fvextra}
\fvset{breaklines=true,breakanywhere=true,fontsize=\footnotesize}
\usepackage{etoolbox}
\usepackage{longtable}
\newcounter{notationtablecounter}
\usepackage{algorithm}
\usepackage{tikz}
\usepackage{pgfplots}
\pgfplotsset{compat=1.18}
\usepgfplotslibrary{groupplots}
\usetikzlibrary{arrows.meta,calc,decorations.pathreplacing,positioning}
\graphicspath{{figs/}}

\definecolor{tolblue}{HTML}{35618E}   
\definecolor{tolcyan}{HTML}{7CA6CC}   
\definecolor{tolgreen}{HTML}{4E8C63}  
\definecolor{tolyellow}{HTML}{B08A3C} 
\definecolor{tolred}{HTML}{B0553F}    
\definecolor{tolpurple}{HTML}{6E5A86} 
\definecolor{tolgrey}{HTML}{8C8C8C}   
\definecolor{tolsand}{HTML}{CBB884}   
\definecolor{tolteal}{HTML}{4E8C7C}   
\pgfplotsset{
  cbstyle/.style={
    tick label style={font=\footnotesize},
    label style={font=\small},
    title style={font=\small, yshift=1pt},
    legend style={font=\footnotesize, draw=gray!40, fill=white, fill opacity=0.85,
                  text opacity=1, row sep=-1pt, inner sep=2pt},
    legend cell align=left,
    grid=major, grid style={gray!18, line width=0.3pt},
    axis line style={gray!55}, tick style={gray!55},
    every axis plot/.append style={line width=0.9pt, mark size=2pt},
  },
}

\newcommand{\missfunc}{\texttt{miss\_func}}
\newcommand{\missparam}{\texttt{miss\_param}}

\newcommand{\methodname}{Critical-State RL}
\newcommand{\gemmamodel}{Gemma-4-26B-A4B}
\newcommand{\gemmasuffix}{26B-A4B}

\usepackage{microtype}
\title{\texorpdfstring{\methodname:\\[0.25em]{\fontsize{14}{18}\selectfont Diagnosing Trainable States for Multi-Turn Tool Use}}{Critical-State RL: Diagnosing Trainable States for Multi-Turn Tool Use}}
\author{\texorpdfstring{%
Zixiang Chen, Wenting Zhao, Zhepeng Cen, Akshara Prabhakar, Jielin Qiu,\\[0.15em]
Jianguo Zhang, Zhiwei Liu, Tulika Manoj Awalgaonkar, Liangwei Yang,\\[0.15em]
Shelby Heinecke, Silvio Savarese, Huan Wang\\[0.55em]
{\small Salesforce AI Research}
}{Zixiang Chen, Wenting Zhao, Zhepeng Cen, Akshara Prabhakar, Jielin Qiu, Jianguo Zhang, Zhiwei Liu, Tulika Manoj Awalgaonkar, Liangwei Yang, Shelby Heinecke, Silvio Savarese, Huan Wang}}
\date{}

\begin{document}
\maketitle

\begin{abstract}
Multi-turn tool-use failures can hinge on a single model call, yet reward variation alone does not reveal which call would benefit from training. When rewards depend on later interactions, their variation can reflect downstream randomness rather than differences between the current actions. We introduce \methodname{} to identify trainable states in multi-turn interactions. Given task-defined candidate calls and local rewards, the method assesses whether each reward captures the action's effect on task success and whether improvement over a reference policy is possible. It then uses nested sampling to separate action-dependent reward variation from continuation noise and optimizes the policy at the selected states using contextual-bandit training.
Experiments on the Berkeley Function Calling Leaderboard (BFCL) v4 compare training at diagnostic-selected states with training at alternative states. For missing-function tasks, the diagnostic selects the response after the tool becomes available; for missing-argument tasks, it selects the response before the missing argument is supplied. Training the selected responses improves performance, including about $14$ percentage points on the missing-function task, while training the alternatives leaves performance flat or worse.
We further apply the recipe across models and tasks, including logged repeat-call avoidance and memory management.
\end{abstract}

\section{Introduction}\label{sec:intro}

Multi-turn tool use involves a sequence of model decisions, yet a trajectory-level reward does not identify which decisions would benefit from training. For example, when a required tool becomes available partway through an interaction, the responses before and after its availability are both candidate training targets. Which response benefits from training can vary by model and task.

Prior work studied turn-level credit assignment~\cite{feng2025gigpo} and state selection for local training~\cite{pivotrl2026,li2026cso}. We focus on whether a candidate call offers an action-dependent signal for training. Observing both successful and failed rollouts does not establish this: when rewards depend on downstream interactions, the same current action can be followed by success or failure. Reward variation then mixes differences between actions with randomness in later interactions. The relevant question is whether changing the current action changes its expected reward.

We introduce \methodname{} to diagnose this signal before training. Given task-defined candidate model calls and a local reward, we assess whether the reward captures the action's effect on task success and whether the reference policy leaves room for improvement. On the model's own rollouts, we sample alternative actions at a fixed context, then hold each action fixed while resampling its continuation. This nested sampling separates action-dependent reward variation from continuation noise. The diagnostic ranks qualifying candidates by this signal. We then optimize the policy at the selected call using contextual-bandit training. Surrounding calls provide context or reward without receiving gradient (Figure~\ref{fig:units}).

We test the diagnostic's turn selections in a four-cell Gemma study on BFCL~\cite{patil2025bfcl}, training selected and alternative turns separately in two task categories (\S\ref{sec:caseA}). The diagnostic selects the response after tool availability for missing-function tasks (\missfunc{}) and the response before missing arguments arrive for missing-argument tasks (\missparam{}). Training the selected turns improves performance, while their alternatives stay flat or decline (\S\ref{sec:caseA-predictor}). Under paired deterministic evaluation, the \missfunc{} arm records \mbox{$0.14\!\to\!0.283\!\pm\!0.015$} across four seeds; the \missparam{} arm records $0.435\!\to\!0.473\!\pm\!0.010$.

The same training recipe applies to logged repeat-call avoidance, Nemotron missing-function interactions, and memory management (\S\ref{sec:applications}). Nemotron trains the decision before tool availability rather than Gemma's recovery response (\S\ref{sec:caseA-results}), illustrating why the training location is chosen for each setting.

\paragraph{Contributions.}
\begin{itemize}[leftmargin=*,labelindent=0pt,labelsep=0.5em]
  \item \textbf{A training-free diagnostic for selecting model calls.} We separate action-dependent reward variation from continuation noise before training. The theory links this signal to unrestricted local-label improvement potential under a small KL budget and bounds finite-sample selection errors.
  \item \textbf{Local training guided by the diagnostic.} We test the selected and alternative turns in a controlled four-cell study, apply the recipe across settings (\S\ref{sec:applications}), and establish a sufficient condition for small local updates to improve full-rollout return (Proposition~\ref{prop:local-transfer}).
\end{itemize}

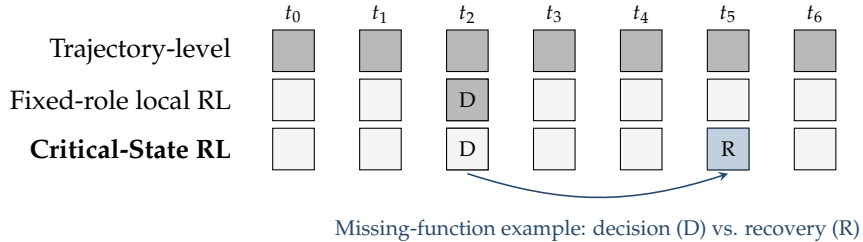
\begin{figure}[H]\centering
\begin{tikzpicture}[x=1.15cm,y=0.65cm,
  cell/.style={draw,minimum size=0.55cm,inner sep=0pt,font=\scriptsize},
  on/.style={cell,fill=gray!55},
  off/.style={cell,fill=gray!8},
  crit/.style={cell,fill=tolblue!30,draw=tolblue!60!black,font=\scriptsize},
  rlab/.style={anchor=east,font=\small}]
  \foreach \i in {0,...,6}{\node[font=\scriptsize] at (\i,1.75) {$t_{\i}$};}
  \node[rlab] at (-0.6,1) {Trajectory-level};
  \foreach \i in {0,...,6}{\node[on] at (\i,1){};}
  \node[rlab] at (-0.6,0) {Fixed-role local RL};
  \foreach \i in {0,...,6}{\node[off] at (\i,0){};}
  \node[on] at (2,0){D};
  \node[rlab] at (-0.6,-1) {\textbf{\methodname{}}};
  \foreach \i in {0,...,6}{\node[off] at (\i,-1){};}
  \node[off] at (2,-1){D};
  \node[crit] at (5,-1){R};
  \draw[->,>=stealth,semithick,tolblue!75!black] (2,-1.5) to[bend right=18] (5,-1.5);
  \node[font=\scriptsize,text=tolblue!75!black] at (3.5,-2.65) {Missing-function example: decision (D) vs. recovery (R)};
\end{tikzpicture}
\caption{Gradient placement in a missing-function example $t_0,\dots,t_6$, with decision (D) before tool availability and recovery (R) after. Trajectory-level training updates every turn; fixed-role local RL always chooses the same role; \methodname{} diagnoses candidate model calls and trains the selected one.}
\label{fig:units}
\readingguide{Shaded $=$ receives gradient. The displayed fixed-role rule chooses decision; the four-cell study also evaluates an always-recovery rule.}
\end{figure}

\long\gdef\reportingprovenancedetails{%
\noindent\textbf{Gemma evaluation protocols.} The four-cell study's selected cells report means and standard deviations across training seeds $\{42,123,7,99\}$. All cells use checkpoint step~30 and paired 200-item BFCL v4 \texttt{multi\_turn} evaluation with \texttt{nt=1} deterministic decoding (\S\ref{sec:caseA-predictor}). The category-specific training study reports recovery-run peaks averaged over five evaluation seeds ($\le\!0.015$ std); the selected no-think Gemma-4-26B-A4B checkpoint records \missfunc{} $0.110\!\to\!0.44$ ($+33$pt; supplementary results below).

\paragraph{Four-cell diagnostic sampling.} Each scenario is one BFCL test item, sampled before training with 8 actions and 4 continuations per action. Table~\ref{tab:diagnostic-verification} reports mean corrected action variance $\widehat{\bar V}_{\mathrm{act}}$ estimated from pre-training samples using the estimator in App.~\ref{app:theory}. The mean includes every clean-prefix scenario, with zero and negative estimates retained. The \missparam{} recovery cell has 54 rather than 64 scenarios because 10 had varying model-generated prefixes; the other three cells each retain 64.

\paragraph{Gemma training setup and supplementary results.}\label{sec:caseA-nothink}
We trained one no-think \gemmamodel{}~\cite{gemmateam2026gemma4} model with category-specific turn assignments: \missfunc{} rows train recovery, while \missparam{} rows train decision. The recovery input comprises the decision context $+$ a sampled-but-frozen refusal $+$ the canonical bridge. Only the sampled recovery action receives gradient; its score checks the held call under the same lenient name$+$argument comparator. The in-domain recovery consequence score rises $0.26\!\to\!0.49\!\to\!0.70$ at steps $0/5/10$. Table~\ref{tab:gemma-nothink} reports BFCL results for the category-routed checkpoint.

The decision-turn pilot retained Nemotron's task, data, agent scaffold, and multiplicative reward (Eq.~\eqref{eq:m2-mult}; App.~\ref{sec:caseA-details}), with thinking disabled rather than chain-of-thought. For \missfunc{}, in-domain home-buying validation remains near $0.21$ (start $0.215$), and BFCL accuracy moves $0.110\!\to\!0.135$. Under $z_{\mathrm{tool}}=\mathrm{no\_write}(a_{\mathrm{dec}})\times\mathrm{consequence}(y_{\mathrm{rec}})$, the decision already passes the no-write gate at a high rate ($\mathrm{no\_write}$ near $0.97$), leaving little headroom in this gate, while $\mathrm{consequence}(y_{\mathrm{rec}})$ varies with recovery success. The sampled-group example in \S\ref{sec:noise-example} shows this variation among clean refusals. Missing-argument examples instead expose premature state-changing calls at the decision turn.

\begin{table}[!ht]
\centering
\small
\caption{Supplementary Gemma results: BFCL v4 target-cell accuracy for the no-think Gemma-4-\gemmasuffix{} study.}
\label{tab:gemma-nothink}
\setlength{\tabcolsep}{10pt}
\renewcommand{\arraystretch}{1.15}
\begin{tabular}{l S[table-format=1.3] S[table-format=1.3]}
\toprule
\textbf{Training configuration} & {Missing function} & {Missing argument} \\
\midrule
Starting checkpoint & 0.110 & 0.446 \\
Decision-turn trained & 0.135 & 0.485 \\
Category-routed, selected step~30 & \textbf{0.439} & 0.478 \\
\bottomrule
\end{tabular}
\floatnote{Starting and category-routed rows are means over 5 evaluation seeds; the decision-turn-trained row is single-eval. The category-routed checkpoint is selected by official BFCL v4 weighted full-suite Overall. Its target-cell mean$\pm$std values are $0.439\!\pm\!0.012$ and $0.478\!\pm\!0.006$.}
\end{table}

Across three independent recovery-turn training runs, peak BFCL missing-function accuracies are $0.359\!\pm\!0.006$, $0.351\!\pm\!0.011$, and $0.439\!\pm\!0.012$, compared with $0.110\!\pm\!0.013$ at the start (mean$\pm$std over 5 evaluation seeds).

\FloatBarrier

\par\smallskip\noindent\textbf{Nemotron missing-function application.} The paired gain is $+4.4$pt under repeated deterministic evaluation ($n{=}3$; Nemotron start $0.408\!\pm\!0.013\to0.452\!\pm\!0.021$) and $+6.5$pt under same-day serving ($0.430\to0.495$). Training on an independently designed synthetic task yields a deterministic $+5.0$pt transfer with a $128$k context window. Nemotron-Super-120B checkpoint trajectories use single evaluations; paired serving checks support the reported gains.

\par\smallskip\noindent\textbf{Memory applications.} The selected xLAM checkpoint lies in a cluster of about $6$--$7$ similarly performing checkpoints; the Gemma application evaluates five data-mixture runs (App.~\ref{sec:caseB}).
}

\section{Related work}\label{sec:related}

\par\noindent\textbf{State selection for local training.}
PivotRL reused expert trajectories collected for supervised fine-tuning (SFT), profiled their states under a frozen reference policy, and retained low-mean states with nonzero verifier variance for group relative policy optimization (GRPO)~\cite{pivotrl2026}. Its verifier $r(s,a)$ rewarded locally acceptable actions rather than exact matches to demonstrations. Critical Step Optimization (CSO) identified candidate steps in failed policy trajectories with a process reward model, verified expert-proposed alternatives through policy continuations, and trained on the resulting step-level preference pairs~\cite{li2026cso}. \methodname{} starts from task-defined candidate model calls and local labels, then separates action-dependent variation from continuation noise to select a call for training.

Karino et al.~\cite{karino2020critical} identified critical states by action-wise expected-return variance, using learned Q-values with uniform action weighting to favor exploitation. We apply the same variance-of-conditional-means principle under the base policy, estimating occurrence-local label means by nested fixed-prefix sampling to select a call for local RL. App.~\ref{sec:recurrent} analyzes the conditional recurrence regime relevant here.

\paragraph{Turn-level rewards and credit assignment.}
MT-GRPO used per-turn rewards~\cite{zeng2025mtgrpo}, while process-supervision methods graded intermediate steps~\cite{uesato2022solving,lightman2024lets}. Group-in-Group Policy Optimization (GiGPO) grouped actions from recurring environment states to estimate step-level relative advantages alongside episode-level credit~\cite{feng2025gigpo}. Other approaches redistributed returns~\cite{arjona2019rudder,ren2022rrd}, conditioned credit on future events or counterfactuals~\cite{harutyunyan2019hca,mesnard2021counterfactual}, or used hindsight and fixed-history comparisons, as in TCPO~\cite{liao2026tcpo}. ArCHer, AgentPRM, and SWEET-RL learned value or reward models~\cite{zhou2024archer,xi2025agentprm,zhou2025sweetrl}. TRACE derived tool-boundary temporal-difference rewards~\cite{tao2026trace}, CAST obtained turn advantages from game solvers~\cite{wang2026cast}, and CrEST combined turn-level verified advantages with token-level modulation~\cite{wang2026crest}.

Executed-replay auditing estimated policy-conditional step contributions by resampling actions and continuing execution~\cite{zhang2026creditwithoutgroundtruth}; LOTAPO replaced a turn and its retrieval observation with a fixed placeholder to estimate its effect~\cite{zhu2026lotapo}. At the execution-interface level, Agent Lightning exposed agent executions as training transitions~\cite{luo2025agentlightning}; v1.0 formalized a harness-owned interaction loop in which the trainer observed large language model (LLM) request/response pairs~\cite{he2026agentlightningv1}.

\paragraph{Training objectives after state selection.}
RLSTA trained the final response with a reward anchored in the model's full-information single-turn likelihood~\cite{chen2026rlsta}, whereas \methodname{} diagnoses which call to train. For a selected call, SFT learns from target actions, direct preference optimization (DPO) from preference pairs~\cite{rafailov2023dpo}, and rejection-sampling fine-tuning from the highest-reward sample in a candidate batch~\cite{touvron2023llama2}. These choices concern how to train once a location is selected. Our targeted-SFT comparison (\S\ref{sec:caseA-predictor}) evaluates this post-selection training stage.

\section{\methodname{}}
\label{sec:recipe}

\methodname{} first selects which model call to train, then updates that call using a local reward. We call a concrete model call an \emph{occurrence} and its local reward an \emph{occurrence-local label}. \textbf{The selected occurrence, not its trajectory, is the training unit:} surrounding calls supply context or reward without receiving gradient.

A \emph{candidate phase} groups contexts with a common task structure, such as a needed tool being unavailable. A phase qualifies as a \emph{critical state} for a model and task configuration when its return-aligned label meets three conditions at every retained occurrence: the label mediates the relationship between action and benchmark return (action-sufficiency); some action's label mean exceeds the reference-policy mean (headroom); and action-conditioned label means vary under the base policy (trainability). Task structure supplies the phase, label, and sufficiency argument; reference- and base-policy rollouts measure headroom and trainability. Figure~\ref{fig:recipe} separates this workflow, tested in \S\ref{sec:caseA}, from task-specific label construction (App.~\ref{sec:caseB}).

\paragraph{A mixed-reward group does not imply trainability.}\label{sec:background}
GRPO~\cite{shao2024deepseekmath} forms a critic-free baseline from each prompt's completion-group rewards. DAPO-style dynamic sampling~\cite{yu2025dapo} filters out binary-correctness groups whose sampled outputs are all correct or all incorrect. By the standard law of total variance, the split $\operatorname{Var}(z\mid x)=\operatorname{Var}_{a}Q(x,a)+\mathbb E_a\!\operatorname{Var}(z\mid x,a)$, where \emph{Q} is the label mean given context and action, separates action-dependent label variance from downstream noise~\cite{karino2020critical}. DAPO's mixed-outcome condition reflects $\operatorname{Var}(z\mid x)>0$ at the sample-group level. Even identical reward marginals and headroom can conceal different action signals (App.~\ref{app:theory}, Figure~\ref{fig:diagnostic-marginals}).

\paragraph{A concrete example: clean refusals with mixed rewards.}\label{sec:noise-example}
In the no-think Gemma \missfunc{} pilot, a sampled mixed-reward prompt group makes this distinction visible. In every surviving rollout, the decision response is a clean refusal with no tool call, yet reward swings $0\!\leftrightarrow\!1$ with within-group std $0.477$. A mixed-outcome filter would retain this group, although its reward differences do not distinguish the observed decision behavior. App.~\ref{sec:caseA-nothink} documents the pilot's training setup and results.

\paragraph{Nested within-prefix sampling.}
Without updating model parameters, we estimate both terms by nested sampling on the model's own rollout distribution: sample a policy prefix, draw candidate actions at that fixed prefix, then hold each action fixed while resampling its reward-only continuation. Action means and within-action variation separate signal from continuation noise; App.~\ref{app:theory} derives the finite-sampling correction and bounds ranking errors.

Ground-truth teacher-forced prefixes can mask recovery-turn trainability by saturating the action distribution at the candidate occurrence. Pooling raw labels across model-generated prefixes mixes action variation with upstream randomness. After candidate occurrences are screened for action-sufficiency and headroom, the diagnostic ranks qualifying occurrences under a common label by $\arg\max_{\text{turn }t}\operatorname{Var}_{a_t}\mathbb E[z\mid x_t,a_t]$; App.~\ref{sec:caseA-nothink} shows that different configurations can select different turns. The repeat-call action-only label and the Gemma continuation-noise case are documented in Table~\ref{app:web-repeat} and App.~\ref{sec:caseA-nothink}.

Under a small KL budget, action-dependent variance governs the best unrestricted local-label improvement to leading order at a fixed prefix and continuation kernel (Theorem~\ref{thm:diagnostic:gain}).

\begin{figure}[H]\centering
\begin{tikzpicture}[
  font=\footnotesize,
  box/.style={rounded corners=3pt, draw=black!35, line width=0.5pt,
              text width=2.55cm, minimum height=2.8cm, inner sep=5pt, align=center},
  core/.style={box, draw=tolblue!65!black, line width=1.1pt},
  arr/.style={->,>=stealth, line width=1pt, draw=black!45},
  alab/.style={font=\scriptsize\itshape, text=black!55, fill=white, inner sep=1pt}]
  \node[box, fill=tolpurple!14] (consequence) at (0,0){\textbf{Candidate model calls}\\[4pt]{\scriptsize task-defined phases and rollout contexts}};
  \node[box, fill=tolteal!18]  (reward) at (3.55,0)  {\textbf{Local reward}\\[4pt]{\scriptsize action-sufficient\\label from:\\action, continuation,\\or proxy}};
  \node[core, fill=tolblue!12] (diagnosis) at (7.1,0) {{\color{tolblue!70!black}\scriptsize(training-free)}\\[1pt]\textbf{Select a call}\\[4pt]{\scriptsize reference-policy\\headroom;\\action-dependent\\label variance}};
  \node[core, fill=tolyellow!30] (isolation) at (10.65,0) {\textbf{Occurrence-local RL}\\[4pt]{\scriptsize gradient on\\selected call only;\\surrounding calls\\receive no gradient}};
  \draw[arr] (consequence) -- (reward);
  \draw[arr] (reward) -- (diagnosis);
  \draw[arr] (diagnosis) -- (isolation);
  \tikzset{groupbrace/.style={decorate,decoration={brace,amplitude=5pt,raise=3pt},draw=black!40,semithick},
            grouplabel/.style={midway,above=7pt,font=\footnotesize\itshape,text=black!55}}
  \draw[groupbrace]
    (consequence.north west) -- (consequence.north west -| reward.north east)
    node[grouplabel]{task-specific label design};
  \draw[groupbrace]
    (diagnosis.north west) -- (diagnosis.north west -| isolation.north east)
    node[grouplabel]{core workflow};
\end{tikzpicture}
\caption{Task structure supplies the candidate phase and occurrence-local label; policy rollouts then measure headroom and action-dependent variance before local training. The BFCL four-cell study in \S\ref{sec:caseA} tests this measured selection step in Gemma; App.~\ref{sec:caseB} presents a contrasting label construction in a memory application.}
\label{fig:recipe}
\end{figure}
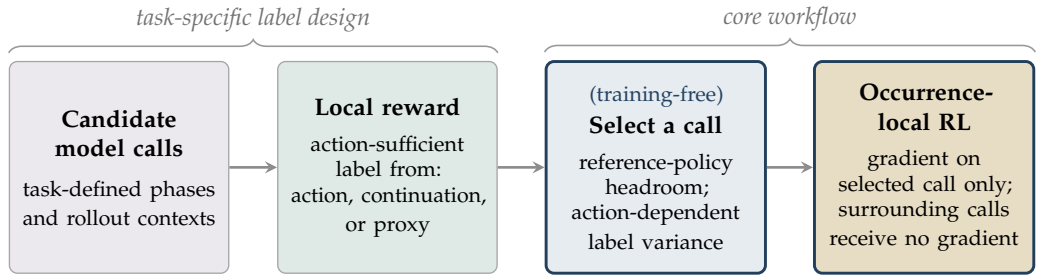

\long\gdef\criticalstateformalism{%
\subsection{When a candidate phase is a critical state}
\label{sec:recipe:formal}
Fix a policy $\pi_\theta$ and terminal return $R(\tau)\in[-R_{\max},R_{\max}]$. A phase $\sigma$ is an equivalence class of contexts, such as missing-function defer/recover, rather than a single fixed context. After routing, an episode has $K\ge1$ retained occurrences, each with context $x_i$, action $a_i$, and score $g_i=\nabla_\theta\log\pi_\theta(a_i\mid x_i)$. Its occurrence-local segment includes a bounded reward-only continuation.

\begin{definition}[Critical state]\label{def:cs}
Relative to a policy $\pi_\theta$ and a reference $\pi_0$, $\sigma$ is a \emph{critical state} if it admits an occurrence-local label $z_i$, supplied by the environment or an oracle/proxy and measurable within at most $h$ turns ($h\ll$ the episode length). With $Q(x,a):=\mathbb E[z\mid x,a]$, every retained context satisfies
\[
\underbrace{R \perp a_i \mid (x_i,z_i)}_{\text{(i) action-sufficiency}},\qquad
\underbrace{\max_a Q(x,a)-\mathbb E_{a\sim\pi_0}Q(x,a)\ge \Delta_{\mathrm{head}}>0}_{\text{(ii) headroom}},\qquad
\underbrace{\operatorname{Var}_{a\sim\pi_\theta(\cdot\mid x)}Q(x,a)>0}_{\text{(iii) trainability}} .
\]
\end{definition}

The headroom margin is task-specific. Contexts failing either (ii) or (iii) are split or filtered before ranking. The bounded label window prevents the vacuous choice $z:=R$ for any state. Benchmark-directed training also requires a return-aligned label; Theorem~\ref{thm:rb} gives the positive-slope condition for per-prompt gradient alignment. For no-think Gemma, $\mathrm{no\_write}$ near $0.97$ leaves little headroom in this gate; the pilot's training comparison favors recovery (App.~\ref{sec:caseA-nothink}).

For the missing-function decision, the full label is $z=\mathrm{no\_write}(a_{\mathrm{dec}})\times\mathrm{consequence}(y_{\mathrm{rec}})$, denoted $z_{\mathrm{tool}}$ in Eq.~\eqref{eq:m2-mult}; $R$ remains the benchmark terminal return. This operational proxy checks the held call and the decision's process gate; the benchmark also checks extraneous recovery calls. Exact sufficiency is the theoretical condition on the whole label. Sampled recovery makes $z$ stochastic even after the decision action is fixed.

\subsection{From the diagnostic to local improvement}
\paragraph{Action-dependent variance measures local improvement potential.}
Fix a prefix $x$ and freeze the continuation kernel that supplies $z$. Write $p(a)=\pi_\theta(a\mid x)$, $Q(a)=\mathbb E[z\mid x,a]$, $q=\mathbb E_p Q$, and $V_{\mathrm{act}}=\operatorname{Var}_p Q$. The score is $g(a)=\nabla_\theta\log p(a)$, with $\mathbb E_p g=0$; vector variance denotes trace covariance.

\begin{proposition}[Continuation noise and the occurrence gradient]\label{prop:diagnostic:noise}
For a fixed baseline $b=b(x)$ and finite second moments,
\begin{align*}
\mathbb E[g(z-b)]&=\mathbb E_p[g(Q-q)]=:\mu,\\
\operatorname{Var}\!\big(g(z-b)\big)
&=\operatorname{Var}_p\!\big(g(Q-b)\big)
+\mathbb E_p\!\left[\|g\|^2\operatorname{Var}(z\mid x,a)\right].
\end{align*}
Thus $V_{\mathrm{act}}=0$ implies $\mu=0$, even when sampled labels vary.
\end{proposition}
\begin{proof}
Condition on $a$ to get mean $g(Q-b)$ and variance $\|g\|^2\operatorname{Var}(z\mid x,a)$. Total expectation and variance, with $\mathbb E_p g=0$, give the identities.
\end{proof}

\par\smallskip\noindent
\begin{minipage}{\linewidth}
\centering
\begin{tikzpicture}[font=\small,y=0.85cm,
  mean cell/.style={draw=tolgrey!60,minimum width=0.68cm,minimum height=0.43cm,inner sep=0pt}]
\node[font=\small\bfseries] at (2.9,1.20) {Action-conditioned means $Q(a)$};
\foreach \x/\i in {1.4/1,2.4/2,3.4/3,4.4/4}
  \node at (\x,0.80) {$a_{\i}$};
\node at (5.8,0.80) {$V_{\mathrm{act}}$};
\node[anchor=east] at (0.75,0.37) {Candidate 1};
\node[anchor=east] at (0.75,-0.16) {Candidate 2};
\foreach \x/\shade/\value in {1.4/0/0,2.4/0/0,3.4/28/1,4.4/28/1}
  \node[mean cell,fill=tolblue!\shade] at (\x,0.37) {$\value$};
\foreach \x/\shade/\value in {1.4/0/0,2.4/14/{\tfrac12},3.4/14/{\tfrac12},4.4/28/1}
  \node[mean cell,fill=tolblue!\shade] at (\x,-0.16) {$\value$};
\node[text=tolblue] at (5.8,0.37) {$1/4$};
\node[text=tolblue] at (5.8,-0.16) {$1/8$};
\draw[tolgrey] (6.35,0.37) -- (7.15,0.10);
\draw[tolgrey] (6.35,-0.16) -- (7.15,0.10);
\draw[-{Stealth[length=5pt]},tolgrey] (7.15,0.10) -- (8.45,0.10);
\node[font=\footnotesize,align=center,text=tolgrey] at (7.50,0.72) {Average\\over actions};
\node[font=\small\bfseries] at (10.2,1.20) {Same label distribution};
\draw[draw=tolgrey,fill=tolgrey!25] (9.25,-0.20) rectangle (9.95,0.38);
\draw[draw=tolgrey,fill=tolgrey!25] (10.45,-0.20) rectangle (11.15,0.38);
\node at (9.60,0.68) {$1/2$};
\node at (10.80,0.68) {$1/2$};
\node[below=2pt] at (9.60,-0.20) {$z=0$};
\node[below=2pt] at (10.80,-0.20) {$z=1$};
\end{tikzpicture}
\captionsetup{font=footnotesize}

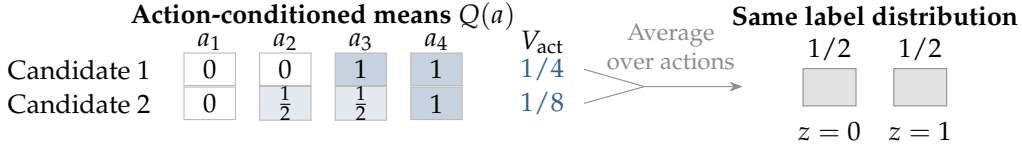
\captionof{figure}{\textbf{Same marginal labels, different action signal.} With uniform base and reference policies, $(0,0,1,1)$ and $(0,\tfrac12,\tfrac12,1)$ are conditional Bernoulli means. Both give $\operatorname{Bernoulli}(1/2)$ labels and $1/2$ headroom, but $V_{\mathrm{act}}=1/4$ and $1/8$.}
\label{fig:diagnostic-marginals}
\end{minipage}\par

No selector using only independent marginal labels and headrooms can distinguish swapped candidates, whose rankings reverse. Its worst-case error is at least $1/2$, at any sample size.

The small-KL sensitivity expansion~\cite{lam2016sensitivity} gives the optimization meaning below; its exponential-tilt optimizer also appears in relative-entropy policy search~\cite{peters2010reps}.
\begin{theorem}[Local improvement at a fixed KL budget]\label{thm:diagnostic:gain}
Let $p$ have positive mass on a finite action set, and allow any new distribution $p'$ on that same set. Holding $Q$ fixed, define
\[
\mathcal G_x(\epsilon):=\max_{D_{\mathrm{KL}}(p'\|p)\le\epsilon}
\big(\mathbb E_{p'}Q-\mathbb E_pQ\big).
\]
If $V_{\mathrm{act}}>0$, then, as $\epsilon\downarrow0$,
\[
\mathcal G_x(\epsilon)=\sqrt{2\epsilon V_{\mathrm{act}}}+O(\epsilon).
\]
If $V_{\mathrm{act}}=0$, then $\mathcal G_x(\epsilon)=0$ for every $\epsilon$.
\end{theorem}
\begin{proof}
Set $A=Q-q$ and $\psi(\eta)=\log\mathbb E_p e^{\eta A}$. Exponential tilting gives $p_\eta(a)=p(a)e^{\eta A(a)-\psi(\eta)}$, with gain $\psi'(\eta)$ and KL $\eta\psi'(\eta)-\psi(\eta)$. For any $p'$,
\[
D_{\mathrm{KL}}(p'\|p_\eta)
=D_{\mathrm{KL}}(p'\|p)-\eta\mathbb E_{p'}A+\psi(\eta)\ge0.
\]
Consequently, $p_\eta$ maximizes gain at its own KL budget. Since $\psi(\eta)=V_{\mathrm{act}}\eta^2/2+O(\eta^3)$, choosing that budget to equal $\epsilon$ yields $\eta=\sqrt{2\epsilon/V_{\mathrm{act}}}+O(\epsilon)$ and the stated expansion. When $V_{\mathrm{act}}=0$, $Q$ is constant on the action set.
\end{proof}

\paragraph{Average improvement over model-generated prefixes.}
Fix a finite-support prefix distribution $\rho$, set $\bar V_{\mathrm{act}}=\mathbb E_\rho V_{\mathrm{act}}(x)$, and assume Theorem~\ref{thm:diagnostic:gain}'s conditions at each prefix. With continuation values fixed, the largest average local-label gain over unrestricted conditional policies is
\[
\mathcal G_\rho(\epsilon)=\sqrt{2\epsilon\bar V_{\mathrm{act}}}+O(\epsilon),
\qquad\mathbb E_\rho D_{\mathrm{KL}}(p'_x\|p_x)\le\epsilon.
\]
This holds for $\bar V_{\mathrm{act}}>0$ as $\epsilon\downarrow0$; zero variance gives zero gain at every budget. Normalize the preceding tilt separately at each prefix and set $\Psi(\eta)=\mathbb E_\rho\log\mathbb E_{p_x}e^{\eta[Q(x,a)-q(x)]}$. Averaging the KL identity proves optimality; $\Psi''(0)=\bar V_{\mathrm{act}}$ gives the expansion. Prefix-mean differences supply no gain because $\rho$ is fixed.

\paragraph{The part accessible to a parameterized policy.}
At a fixed prefix, let $F=\mathbb E_p[gg^\top]$, let $F^\dagger$ be its Moore--Penrose inverse, and define $V_{\mathrm{acc}}=\mu^\top F^\dagger\mu$. Compatible-function projection and the quadratic KL model~\cite{kakade2001natural} give
\begin{align*}
V_{\mathrm{act}}&=V_{\mathrm{acc}}+
\mathbb E_p\!\left[(Q-q-g^\top F^\dagger\mu)^2\right],\\
\max_{u^\top Fu/2\le\epsilon}\mu^\top u
&=\sqrt{2\epsilon V_{\mathrm{acc}}}.
\end{align*}
Indeed, $\mu$ belongs to the range of $F$, and $g^\top F^\dagger\mu$ is the least-squares projection of $Q-q$ onto the score span. Orthogonality proves the first identity; Cauchy--Schwarz in the $F$ metric proves the second. A full categorical policy has $V_{\mathrm{acc}}=V_{\mathrm{act}}$.

Under a common label and equal small KL budgets, $V_{\mathrm{act}}$ ranks leading-order unrestricted potential at a fixed prefix. Across a fixed distribution of prefixes, a shared update under an average quadratic KL budget uses $\bar\mu=\mathbb E_x\mu_x$ and $\bar F=\mathbb E_xF_x$, with potential $\sqrt{2\epsilon\bar\mu^\top\bar F^\dagger\bar\mu}$. Applying the same projection under the joint law $\rho(x)p_x(a)$ gives $\bar\mu^\top\bar F^\dagger\bar\mu\le\bar V_{\mathrm{act}}$. Opposing prefix gradients can cancel. The diagnostic is gradient-free; the four-cell intervention tests the shared-policy response.

\subsection{Finite-sample estimation and occurrence selection}
For a balanced nested design at a fixed prefix, draw $n\ge2$ independent actions and $m\ge2$ continuations per action, all independent conditional on the actions. Let $S_{\mathrm{between}}^2$ be the sample variance (denominator $n-1$) of the $n$ action means, and $S_{\mathrm{within}}^2$ the average within-action sample variance (denominator $m-1$). Writing $V_{\mathrm{cont}}=\mathbb E_a\operatorname{Var}(z\mid x,a)$ gives
\[
\mathbb E S_{\mathrm{between}}^2=V_{\mathrm{act}}+\frac{V_{\mathrm{cont}}}{m},
\qquad
\widehat V_{\mathrm{act}}=S_{\mathrm{between}}^2-\frac{S_{\mathrm{within}}^2}{m},
\qquad
\mathbb E\widehat V_{\mathrm{act}}=V_{\mathrm{act}}.
\]
Total variance of an action mean proves the first equality; $\mathbb E S_{\mathrm{within}}^2=V_{\mathrm{cont}}$ gives the correction. This standard nested-simulation correction~\cite{sun2011nested,goda2017conditional} removes continuation noise in expectation; finite estimates can be negative.

\paragraph{From an unbiased estimate to reliable occurrence selection.}
Classical nested-variance analysis separates action coverage from continuation depth~\cite{sun2011nested}. Here that distinction determines how reliably the diagnostic ranks candidate occurrences.

\begin{proposition}[Finite-sample diagnostic precision]\label{prop:diagnostic:precision}
Under the preceding balanced design with bounded labels, let $d(a)=Q(a)-q$ and $\nu(a)=\operatorname{Var}(z\mid x,a)$. Then
\begin{align}
\mathcal E_{n,m}:=\operatorname{Var}(\widehat V_{\mathrm{act}})
&=\frac{\operatorname{Var}_a(d^2)}{n}
+\frac{4\mathbb E_a[d^2\nu]}{nm}
+\frac{2\mathbb E_a[\nu^2]}{nm(m-1)}\nonumber\\
&\quad+\frac{2(V_{\mathrm{act}}+V_{\mathrm{cont}}/m)^2}{n(n-1)}.
\label{eq:diagnostic:precision}
\end{align}
Consider two qualifying candidate occurrences at fixed prefixes under a common label scale, estimated using independent nested batches. If $V_1>V_2$, write $\Delta=V_1-V_2$ and let $\mathcal E_1,\mathcal E_2$ be their estimator variances from \eqref{eq:diagnostic:precision}. Their misranking probability satisfies
\[
\Pr(\widehat V_2\ge\widehat V_1)
\le\frac{\mathcal E_1+\mathcal E_2}{\Delta^2+\mathcal E_1+\mathcal E_2}.
\]
In particular, $\mathcal E_1+\mathcal E_2\le\delta\Delta^2/(1-\delta)$ guarantees correct selection with probability at least $1-\delta$, for $0<\delta<1$.
\end{proposition}
\begin{proof}
Center labels at $q$. For action $i$, let $Y_i=m^{-1}\sum_j(z_{ij}-q)$, let $S_i^2$ be its within-action sample variance, and set
\[
T_i=Y_i^2-S_i^2/m
=\frac{\sum_{j\ne k}(z_{ij}-q)(z_{ik}-q)}{m(m-1)}.
\]
Conditional independence gives $\mathbb E[T_i\mid a_i]=d(a_i)^2$ and $\operatorname{Var}(T_i\mid a_i)=4d(a_i)^2\nu(a_i)/m+2\nu(a_i)^2/[m(m-1)]$. Algebra rewrites the existing estimator as
\[
\widehat V_{\mathrm{act}}=\frac1n\sum_iT_i
-\frac{\sum_{i\ne k}Y_iY_k}{n(n-1)}.
\]
Since the groups are independent and $\mathbb E Y_i=0$, the two terms are uncorrelated; the second has variance $2(\mathbb E Y_i^2)^2/[n(n-1)]$. Total variance for $T_i$ and $\mathbb E Y_i^2=V_{\mathrm{act}}+V_{\mathrm{cont}}/m$ prove \eqref{eq:diagnostic:precision}. Finally, $\widehat V_2-\widehat V_1$ has mean $-\Delta$ and variance $\mathcal E_1+\mathcal E_2$; the one-sided Chebyshev (Cantelli) inequality gives the ranking bound. For more candidates, sum the pairwise bounds against the unique best candidate.
\end{proof}

\paragraph{More continuations cannot replace more actions.}
As $m\to\infty$ at fixed $n$, continuation noise vanishes but the estimator variance tends to $\operatorname{Var}_a(d^2)/n+2V_{\mathrm{act}}^2/[n(n-1)]$, which is positive when $V_{\mathrm{act}}>0$. Conversely, any fixed $m\ge2$ allows consistent estimation as $n$ grows, and the misranking bound vanishes for a fixed positive gap. Thus independent actions can resolve the ranking without precise estimates of every action's label mean.

\paragraph{Averaging estimates across sampled prefixes.}
For $P$ independent prefixes $X_\ell\sim\rho$ with independent balanced $(n,m)$ batches, set $\widehat{\bar V}_{\mathrm{act}}=P^{-1}\sum_{\ell=1}^P\widehat V_{\mathrm{act}}(X_\ell)$. Conditional unbiasedness and total variance give
\[
\mathbb E\widehat{\bar V}_{\mathrm{act}}=\bar V_{\mathrm{act}},\qquad
\operatorname{Var}(\widehat{\bar V}_{\mathrm{act}})
=\frac{\operatorname{Var}_\rho(V_{\mathrm{act}}(X))+\mathbb E_\rho\mathcal E_{n,m}(X)}{P}.
\]
The same misranking bound applies to independent candidate datasets using these averaged targets and variances. More actions or continuations reduce within-prefix error; more independent prefixes also reduce between-prefix error.

\subsection{From local-label credit to terminal-return credit}
Proposition~\ref{prop:diagnostic:noise} isolates noise \emph{within} the local label. The next result concerns a different source: terminal-return variation \emph{outside} that label.

\begin{theorem}[Per-prompt conditioning on an occurrence-local label]\label{thm:rb}
Let $\hat g^{\mathrm{traj}}_i=g_i(R-b(x_i))$, with $b(x_i)=\mathbb E[R\mid x_i]$, and $\hat g^{\mathrm{flat}}_i=g_i(z_i-q(x_i))$, with $q(x_i)=\mathbb E[z\mid x_i]$. Set $\mathcal F_i=\sigma(x_i,a_i,z_i)$. If $R\perp a_i\mid(x_i,z_i)$ and $\varphi_x(z):=\mathbb E[R\mid x,z]$ is affine with slope $\lambda(x)$, then
\[
\mathbb E[\hat g^{\mathrm{traj}}_i\mid\mathcal F_i]=\lambda(x_i)\hat g^{\mathrm{flat}}_i,
\qquad
\operatorname{Var}(\hat g^{\mathrm{traj}})
=\operatorname{Var}(\lambda(x)\hat g^{\mathrm{flat}})
+\mathbb E[\|g\|^2\operatorname{Var}(R\mid x,a,z)].
\]
\end{theorem}
\begin{proof}
Sufficiency gives $\mathbb E[R\mid x,a,z]=\varphi_x(z)$. Affinity gives $\varphi_x(z)-\mathbb E[R\mid x]=\lambda(x)(z-q(x))$. Multiply by the measurable score and apply total variance.
\end{proof}

At a fixed prompt, $\lambda(x)>0$ preserves gradient direction; across prompts the return gradient weights local gradients by $\lambda(x)$. Binary labels make the link affine automatically; continuous missing-function and memory scores require calibration of the full label. A small $\lambda$ attenuates local credit and a negative one reverses it. If the conditional-mean sufficiency residual is bounded by $\varepsilon$, the conditional-gradient identity has error at most $\varepsilon\|g\|$.

\begin{corollary}[Return residual and independent recurrence]\label{cor:rec}
At a fixed prompt, suppose $R=z+w$ with $w$ independent of $(g,z)$. Both estimators have mean $\mu=\mathbb E[g(z-\mathbb E z)]$, and
\[
\operatorname{Var}(\hat g^{\mathrm{traj}})
=\operatorname{Var}(\hat g^{\mathrm{flat}})+\mathbb E\|g\|^2\operatorname{Var}(w),
\qquad
S_z=\frac{\operatorname{Var}(z)}{\operatorname{Var}(R)}.
\]
For nonzero $\mu$, define the multiplier signal-to-noise ratio (SNR) as $\|\mu\|$ divided by the multiplier's standard deviation. Then $\mathrm{SNR}_{\mathrm{traj}}/\mathrm{SNR}_{\mathrm{flat}}=\sqrt{S_z}$. If $w$ sums $K-1$ independent sibling labels, each with the same positive variance as $z$ and unaffected by the selected action, then $S_z=1/K$. App.~\ref{sec:recurrent} gives the full per-occurrence covariance and sample-cost result. For $K=1$, $w$ can still contain terminal-return variation beyond the local label.
\end{corollary}
\begin{proof}
The centered residual $g(w-\mathbb E w)$ has zero mean and zero covariance with $g(z-\mathbb E z)$. Its trace variance is $\mathbb E\|g\|^2\operatorname{Var}(w)$; scalar variance addition gives $S_z$.
\end{proof}

\paragraph{Empirical single-occurrence scope.}
On Gemma-4-26B-A4B home-buying rollouts ($n{=}6560$), the per-turn-share proxy gives a pooled variance ratio $S_z=\operatorname{Var}(z)/\operatorname{Var}(R)=0.31$ ($95\%$ confidence interval (CI) $[0.29,0.33]$; $\sqrt{S_z}$ is about $0.56$), with slope $\lambda$ near $1.7$ and correlation $0.95$. The field-survey analysis has a signed terminal-reward mean difference of $-0.021$ between correct and wrong loadouts (App.~\ref{sec:rec:evidence}). The ratio and its square root depend on label scale: these are descriptive associations, not measured SNR losses or a fitted $1/\sqrt K$ exponent. The prompt-centered ratio is $\mathbb E[\operatorname{Var}(z\mid x)]/\mathbb E[\operatorname{Var}(R\mid x)]$, equal to the marginal ratio when prompt means do not vary. App.~\ref{sec:recurrent} analyzes group credit; App.~\ref{sec:stopping} analyzes changes to the surrounding policy.
}

\subsection{Core operation: occurrence-local RL}
\label{sec:recipe:m1}
Occurrence-local RL applies its loss to one selected model call per training unit. It supplies context either as a \emph{precomputed prompt} without an environment loop or through a trajectory whose downstream, gradient-free turns provide the consequence signal. The memory applications and no-think Gemma's missing-function recovery training use the first form (App.~\ref{sec:caseA-nothink}); Nemotron's missing-function decision training uses the second (\S\ref{sec:caseA-m1m2}).

\begin{algorithm}[H]
\caption{Diagnose an occurrence, then train it with an occurrence-local label.}
\label{alg:apply}
\small
\noindent\textbf{Input:} task-defined candidate phase and label; base and reference policies held fixed during diagnosis.
\begin{enumerate}[leftmargin=*,itemsep=2pt,topsep=3pt]
\item \textbf{Screen candidates.} Assess action-sufficiency from task structure and check reference-policy headroom for each candidate occurrence.
\item \textbf{Separate signal from noise.} Sample prefixes and actions from the base policy. At each fixed prefix, hold each sampled action fixed while resampling any reward-only continuation. Estimate action means and continuation variances, then apply the finite-sampling correction to estimate action-dependent variance (App.~\ref{app:theory}).
\item \textbf{Select the occurrence.} For each candidate, average the corrected estimates across its sampled prefixes. Among candidates passing the gates under a common label, select the one with the largest average.
\item \textbf{Train locally.} Optimize the policy using the local label, applying the RL loss only to the selected call's generated tokens. Its prefix supplies context; any continuation supplies reward without receiving gradient.
\end{enumerate}
\noindent\textbf{Output:} the trained policy and its selected training location, local label, and gradient boundary.
\end{algorithm}

\subsection{Task-specific reward construction}
\label{sec:recipe:m2}
\label{sec:recipe:m3}
Task structure supplies the local reward used for diagnosis and training. The following applications combine checks on the current action with scores of downstream behavior or preserved information.

\paragraph{Missing-function interaction (Nemotron).}
For Nemotron's missing-function application (\S\ref{sec:caseA-results}), the multiplicative reward in Eq.~\eqref{eq:m2-mult} combines a behavior gate with a score from sampled recovery:
\begin{equation}
z_{\mathrm{tool}} \;=\; \mathrm{no\_write}(a_{\mathrm{dec}})\,\times\,\mathrm{consequence}(y_{\mathrm{rec}}),
\label{eq:m2-mult}
\end{equation}
Here $a_{\mathrm{dec}}$ and $y_{\mathrm{rec}}$ denote the decision action and recovery output. The gate $\mathrm{no\_write}(a_{\mathrm{dec}})\in\{0,1\}$ is $1$ iff the decision makes no state-changing write; $\mathrm{consequence}(y_{\mathrm{rec}})\in[0,1]$ scores whether recovery contains the required tool call. The sampled recovery supplies reward without receiving gradient.

\noindent Here $z_{\mathrm{tool}}$ supplies the occurrence-local label $z$ of App.~\ref{sec:recipe:formal}, while $R$ denotes the benchmark terminal return.

\long\gdef\memoryrewardformuladetails{%
\begingroup\small
\setlength{\abovedisplayskip}{\smallskipamount}\setlength{\belowdisplayskip}{\smallskipamount}%
\par\noindent\textbf{Reward definition.} The additive storage reward grants partial credit:
\[
\mathrm{total} \;=\; \underbrace{\mathrm{keyword\_match\_score} + 0.1\cdot\mathrm{memory\_type\_score}}_{\text{base}} \;+\; \mathrm{behavior\_bonus} \;+\; \mathrm{memory\text{-}full\ bonuses},
\]
Here \texttt{base} is the keyword-and-memory-type subtotal. The total is clipped to $[-1,1]$, with $\mathrm{behavior\_bonus}$ awarding $+0.1$ per correct procedural element (e.g.\ replace, archival add, remove-before-add) and memory-full bonuses adding a $+0.2$ result term and a $\pm0.2$ trash-removal term. Hard gates override this sum: a destructive memory clear returns $-1.0$, and a memory-full naive add-without-remove returns $-0.5$.
\par\endgroup
}

\paragraph{Memory storage.}
In contrast to Eq.~\eqref{eq:m2-mult}, memory uses an additive reward for correct operation order and preservation of answer-critical information, with hard gates for invalid actions. A precomputed keyword proxy grades storage text against 2--5 required terms extracted by GPT-4o during preprocessing~\cite{hurst2024gpt4o}, including meaningful negations such as the \emph{not} in ``not married.'' The proxy measures keyword retention in the storage response; the benchmark evaluates the full storage-to-answer chain.

\long\gdef\coveragemethoddetails{%
\paragraph{Optional coverage audit by slice.}
\label{sec:recipe:m4}
A coverage audit tests whether a structural mismatch between training and evaluation slices explains slice-specific changes. It partitions candidate occurrences along a \emph{structural} axis (a state property, not a reward or hyperparameter), measures training--evaluation mismatch on that axis, and adds rows for the under-covered slice. A \emph{strictly-additive} repair leaves existing rows, the reward, and hyperparameters untouched, while a size-matched control fixes the added row count and varies only the slice. The audit diagnoses a \emph{distribution} gap; trainable-occurrence selection still comes from the diagnostic.
}

\long\gdef\criticalstatepractice{%
Algorithm~\ref{alg:apply} gives the diagnostic and local-training steps. Here we detail how to prepare its task-defined inputs and when to run additional analyses.

\paragraph{Identify a candidate phase when it is not given.}
Absent identifying metadata, cross-tabulate the model's per-sample failures against a stronger reference, then bucket residuals by a structural axis (error type, turn position); a concentration of residuals identifies a candidate phase. Within that phase, select candidate rows from scenario metadata or a deterministic rule. For example, the missing-function application uses an LLM judge to exclude rows that another tool could satisfy.

\paragraph{Transfer monitoring.}
When training against a frozen single-occurrence surrogate, monitor transfer and Kullback--Leibler (KL) drift under the stated conditions (App.~\ref{sec:stopping}).

The selected training location varies with the model and task configuration, while the task determines whether an occurrence-local label exists (App.~\ref{sec:caseA-nothink}, App.~\ref{sec:recurrent}).

\paragraph{Candidate identification and validation.}
\label{sec:recipe:localize}
For the missing-function application, humans proposed the candidate phase and designed the counterfactual validation; deterministic rules selected rows. App.~\ref{app:localize} records this procedure and its results. For memory storage, the memory-full flag supplies candidate membership.
}

\section{BFCL four-cell study of occurrence selection}\label{sec:caseA}

\paragraph{Benchmark setting.}
We test whether the diagnostic identifies which candidate turn benefits from training. Evaluation uses the BFCL~\cite{patil2025bfcl} v4 \emph{multi\_turn} suite: \texttt{base}, \texttt{miss\_func}, \texttt{long\_context}, and \texttt{miss\_param}. In our configuration, the cumulative multi-turn evaluator penalizes extra writes, missing calls, and step-cap violations but not read-only queries.

An earlier no-think Gemma pilot trained the \missfunc{} decision turn and produced a mixed-reward group of clean refusals with no call. That observation motivated the nested variance decomposition. Appendix~\ref{sec:caseA-nothink} gives its training setup and supplementary results; Section~\ref{sec:applications} presents applications in other settings.

\subsection{BFCL task mechanics}\label{sec:caseA-setup}

In BFCL \texttt{multi\_turn} \texttt{miss\_func}, a required tool is withheld until turn $k$. The substantive user request appears at turn $k{-}1$, while the tool is unavailable; turn $k$ carries an \emph{empty} user message that the harness replaces with a fixed bridge announcing tool availability. The agent must therefore \emph{defer} or issue a read-only query at turn $k{-}1$, then emit the held call once the bridge re-introduces it. We call turn $k{-}1$ the \emph{decision turn} and turn $k$ the \emph{recovery turn}. Appendix Figure~\ref{fig:case-a-missfunc} illustrates one \texttt{miss\_func} decision-turn instance. In \texttt{miss\_param}, by contrast, the function is present but a required argument is missing until the following turn; the decision-turn failure is a premature state-changing call.

\subsection{Diagnostic assignments and four-cell intervention}\label{sec:caseA-predictor}

Task mechanics define the candidate occurrences and labels. Using rollouts from frozen no-think Gemma-4-26B-A4B~\cite{gemmateam2026gemma4}, the pre-training diagnostic in \S\ref{sec:recipe} assigns recovery for \missfunc{} and decision for \missparam{}. The subsequent four-cell study fixes these assignments, then trains decision and recovery separately in each category. Table~\ref{tab:diagnostic-verification} pairs the training outcomes with mean corrected action variance: within each category, the selected occurrence has the higher estimate.

\begin{table}[H]
\centering
\small
\caption{Base-policy diagnostics and training outcomes for \methodname{} on BFCL v4 \texttt{multi\_turn}.}
\label{tab:diagnostic-verification}
\setlength{\tabcolsep}{4pt}
\renewcommand{\arraystretch}{1.15}
\begin{tabular}{@{}>{\raggedright\arraybackslash}p{0.33\linewidth}>{\centering\arraybackslash}p{0.15\linewidth}>{\centering\arraybackslash}p{0.28\linewidth}>{\centering\arraybackslash}p{0.12\linewidth}@{}}
\toprule
\textbf{Cell and diagnostic assignment} & \textbf{\shortstack{Mean corrected\\action variance}} & \textbf{\shortstack{BFCL accuracy\\start $\to$ trained}} & $\boldsymbol{\Delta}$ \\
\midrule
\rowcolor{bestrow}
\shortstack[l]{\textbf{\missfunc{} / recovery}\\{\scriptsize selected occurrence}} & $0.0267$ & $0.14\to\mathbf{0.283\pm0.015}$ & $\mathbf{+14.3}$\,pp \\
\shortstack[l]{\missfunc{} / decision\\{\scriptsize alternative occurrence}} & $0.0103$ & $0.14\to0.095$ & $-4.5$\,pp \\
\addlinespace
\shortstack[l]{\textbf{\missparam{} / decision}\\{\scriptsize selected occurrence}} & $0.0361$ & $0.435\to\mathbf{0.473\pm0.010}$ & $\mathbf{+3.8}$\,pp \\
\shortstack[l]{\missparam{} / recovery\\{\scriptsize alternative occurrence}} & $0.0000$ & $0.435\to0.445$ & $+1$\,pp \\
\bottomrule
\end{tabular}
\floatnote{Action variance: finite-sampling-corrected estimates from base-policy samples, averaged across clean-prefix scenarios (App.~\ref{app:prov}). Selected cells: BFCL mean$\pm$cross-seed std, seeds $\{42,123,7,99\}$; alternatives: point estimates. All use deterministic \texttt{nt=1} decoding, step~30, and paired 200-item BFCL v4 \texttt{multi\_turn} evaluation. Bases differ from the category-specific five-evaluation-seed study (App.~\ref{sec:caseA-nothink}). Bold marks selected occurrences; shading marks the primary result. The \missparam{} result is directional.}
\end{table}

Recovery improves \missfunc{} by $14.3\!\pm\!1.5$\,pp but leaves \missparam{} flat at $+1$\,pp; decision improves \missparam{} by $3.8\!\pm\!1.0$\,pp but moves \missfunc{} down by $4.5$\,pp.

\noindent\textbf{Selection-rule comparison.} Table~\ref{tab:diagnostic-verification} compares the diagnostic with always-decision and always-recovery selection under the same occurrence-local RL interface. Each fixed rule improves one category but leaves the other flat or worse; the diagnostic selects the improving turn in both.

\noindent We evaluate all cells at the fixed step~30 checkpoint. In the two diagnostic-selected cells, gym validation is monotone, and BFCL scores from steps 10--30 lie on a stable plateau above base.

\noindent\textbf{Transfer to a novel bridge.} We replace the canonical bridge at evaluation with an unseen wording that names no tools or functions: ``I adjusted things on my end.'' Using the seed-123 step-30 checkpoint with deterministic decoding on 200 paired \missfunc{} examples, accuracy rises from $0.095$ to $0.225$ after recovery-turn RL ($+13.0$\,pp), matching the canonical bridge's $0.140\to0.270$ gain under the same protocol. The new wording is harder for the base policy, yet preserves the recovery gain, demonstrating transfer beyond the bridge wording used in training.

\begin{table}[H]
\centering
\small
\caption{Monte-Carlo return-to-go (RTG) comparison on the non-saturated \missparam{} cell.}
\label{tab:rtg-comparator}
\setlength{\tabcolsep}{7pt}
\renewcommand{\arraystretch}{1.15}
\begin{tabular}{@{}>{\raggedright\arraybackslash}p{0.49\linewidth}>{\centering\arraybackslash}p{0.14\linewidth}>{\centering\arraybackslash}p{0.16\linewidth}>{\centering\arraybackslash}p{0.10\linewidth}@{}}
\toprule
\textbf{Method} & \textbf{Gym val.} & \textbf{BFCL \missparam{}} & $\boldsymbol{\Delta}$ \\
\midrule
\shortstack[l]{Monte-Carlo RTG\\{\scriptsize temporal credit; full multi-turn}} & $0.997$ & $0.45$ & $+1.5$\,pp \\
\rowcolor{bestrow}
\shortstack[l]{\methodname{}\\{\scriptsize contextual-bandit training}} & $0.77$ & $\mathbf{0.473\pm0.010}$ & $\mathbf{+3.8}$\,pp \\
\bottomrule
\end{tabular}
\floatnote{The \methodname{} BFCL result is the four-seed summary from Table~\ref{tab:diagnostic-verification}. RTG is a single-seed comparator and uses Monte-Carlo return-to-go rather than a learned parametric critic.}
\end{table}

\noindent\textbf{Return-to-go comparison.} We compare occurrence-local RL with full multi-turn training using Monte-Carlo return-to-go. RTG reaches a gym validation score of $0.997$ (versus $0.77$ for the bandit), but its BFCL accuracy is $0.45$, below \methodname{}'s $0.473\!\pm\!0.010$ (Table~\ref{tab:rtg-comparator}).

\noindent\textbf{Targeted-SFT comparison.} We also compare occurrence-local RL with single-seed targeted SFT at the selected turn. Targeted SFT trains on fixed action targets, whereas \methodname{} samples actions on-policy and scores them with the verifier. \mbox{On \missfunc{} / recovery}, \methodname{} gains $14.3$\,pp, whereas the best targeted recovery-SFT checkpoint remains $6$\,pp below base. The targeted-SFT runs over-bias toward firing the tool and break the required refusal; the \methodname{} result preserves both behaviors. \mbox{On \missparam{} / decision}, targeted SFT gains $3$\,pp, compared with \methodname{}'s $+3.8$\,pp.

\noindent\textbf{Turn-mask control.} This control varies which turns receive gradient while retaining the trajectory-level scalar $A=R-b$. Full, recovery, random, and decision masks all leave the learnable \missparam{} cell near its starting accuracy: changing the mask alone does not recover the occurrence-local RL gain.

\long\gdef\caseAworkedexample{%
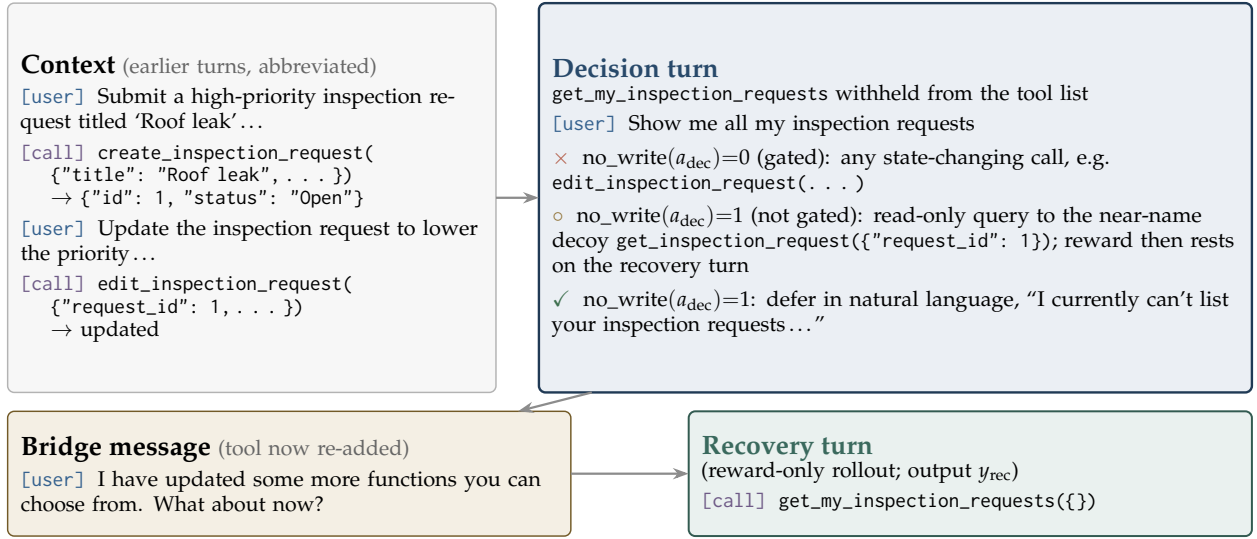
\begin{figure}[H]
\centering
\small
\begin{tikzpicture}[
  flow card/.style={
    rounded corners=2.5pt,
    draw=black!30,
    line width=0.55pt,
    inner sep=5pt,
    outer sep=0pt,
    align=left,
    anchor=north west,
    font=\scriptsize
  },
  context card/.style={flow card, fill=tolgrey!7},
  decision card/.style={
    flow card,
    fill=tolblue!10,
    draw=tolblue!60!black,
    line width=0.9pt
  },
  bridge card/.style={
    flow card,
    fill=tolyellow!14,
    draw=tolyellow!65!black
  },
  recovery card/.style={
    flow card,
    fill=tolteal!12,
    draw=tolteal!65!black,
    line width=0.75pt
  },
  flow arrow/.style={
    -{Stealth[length=2.1mm,width=1.5mm]},
    line width=0.85pt,
    draw=black!45
  }
]
\node[context card, text width=0.37\linewidth, minimum height=5.15cm]
  (context) at (0,0) {%
  {\small\bfseries Context}
  {\color{black!60}(earlier turns, abbreviated)}\\[2pt]
  \textcolor{tolblue}{\texttt{[user]}}\; Submit a high-priority inspection request titled
  `Roof leak'\,\ldots\\[2pt]
  \textcolor{tolpurple}{\texttt{[call]}}\;
  \texttt{create\_inspection\_request(}\\[-1pt]
  \hspace*{1.4em}\texttt{\{"title": "Roof leak",\,\ldots\})}\\[-1pt]
  \hspace*{1.4em}$\to$ \texttt{\{"id": 1, "status": "Open"\}}\\[2pt]
  \textcolor{tolblue}{\texttt{[user]}}\; Update the inspection request to lower the
  priority\,\ldots\\[2pt]
  \textcolor{tolpurple}{\texttt{[call]}}\;
  \texttt{edit\_inspection\_request(}\\[-1pt]
  \hspace*{1.4em}\texttt{\{"request\_id": 1,\,\ldots\})}\\[-1pt]
  \hspace*{1.4em}$\to$ updated
};
\node[decision card, anchor=north east, text width=0.55\linewidth,
      minimum height=5.15cm] (decision) at (0.997\linewidth,0) {%
  {\small\bfseries\color{tolblue!75!black}Decision turn}\\[-1pt]
  \texttt{get\_my\_inspection\_requests} withheld from the tool list\\[2pt]
  \textcolor{tolblue}{\texttt{[user]}}\; Show me all my inspection requests\\[3pt]
  {\color{tolred}\bfseries$\times$}\;
  $\mathrm{no\_write}(a_{\mathrm{dec}}){=}0$ (gated): any state-changing call, e.g.\
  \texttt{edit\_inspection\_request(\ldots)}\\[3pt]
  {\color{tolyellow!80!black}\bfseries$\circ$}\;
  $\mathrm{no\_write}(a_{\mathrm{dec}}){=}1$ (not gated): read-only query to the near-name decoy
  \texttt{get\_inspection\_request(\{"request\_id": 1\})}; reward then rests on the recovery turn\\[3pt]
  {\color{tolgreen!80!black}\bfseries$\checkmark$}\;
  $\mathrm{no\_write}(a_{\mathrm{dec}}){=}1$: defer in natural language,
  ``I currently can't list your inspection requests\,\ldots''
};
\node[bridge card, text width=0.43\linewidth, minimum height=1.65cm]
  (bridge) at (0,-5.40cm) {%
  {\small\bfseries Bridge message}
  {\color{black!60}(tool now re-added)}\\[2pt]
  \textcolor{tolblue}{\texttt{[user]}}\; I have updated some more functions you can choose
  from. What about now?
};
\node[recovery card, anchor=north east, text width=0.43\linewidth,
      minimum height=1.65cm] (recovery) at (0.997\linewidth,-5.40cm) {%
  {\small\bfseries\color{tolteal!75!black}Recovery turn}\\[-1pt]
  (reward-only rollout; output $y_{\mathrm{rec}}$)\\[2pt]
  \textcolor{tolpurple}{\texttt{[call]}}\;
  \texttt{get\_my\_inspection\_requests(\{\})}
};
\draw[flow arrow] (context.east) -- (decision.west);
\draw[flow arrow]
  ([xshift=7mm]decision.south west) -- ([xshift=-7mm]bridge.north east);
\draw[flow arrow] (bridge.east) -- (recovery.west);
\end{tikzpicture}
\caption{A missing-function decision-turn instance in the Nemotron decision-turn configuration (renamed home-buying variant).}
\label{fig:case-a-missfunc}
\floatnote{The gate $\mathrm{no\_write}(a_{\mathrm{dec}})$ scores the decision under the convention in \S\ref{sec:caseA-setup}; $\mathrm{consequence}(y_{\mathrm{rec}})$ scores the held recovery call (Eq.~\eqref{eq:m2-mult}).}
\end{figure}
}

\long\gdef\caseAsetupdetails{%
\paragraph{BFCL-aligned scoring.} The withheld tool is encoded as \texttt{missed\_function}\,$=\{\texttt{"}k\texttt{"}:[\texttt{tool}]\}$. BFCL \texttt{multi\_turn} grading (\texttt{multi\_turn\_checker.py}) is state-based and cumulative: it penalizes (i)~extra state-changing calls, (ii)~missing required calls, and (iii)~per-turn step-cap violations, but not read-only queries.

\paragraph{Training-only home-buying variant.} Training and in-domain validation use a renamed home-buying \emph{task variant} of BFCL \texttt{multi\_turn} with identical stateful tool APIs and turn structure; it is not held out. \emph{The Nemotron benchmark results use the unmodified BFCL v4 \texttt{multi\_turn} benchmark}. Training on the independently designed hospital-ward task yields held-out BFCL transfer (Table~\ref{tab:caseA}).

}

\long\gdef\caseAnemotronrewarddetails{%
\paragraph{Two-axis reward.} The occurrence-local label for a sampled rollout is (Figure~\ref{fig:decision-turn})
\[
z_{\mathrm{tool}} \;=\; \underbrace{\mathrm{no\_write}(a_{\mathrm{dec}})}_{\text{procedural / behavior axis}} \;\times\; \underbrace{\mathrm{consequence}(y_{\mathrm{rec}})}_{\text{content / consequence axis}},
\]
where the process gate is
\[
  \mathrm{no\_write}(a_{\mathrm{dec}}) =
  \begin{cases}
    1, & \text{if the decision-turn action emits no state-changing write,}\\
    0, & \text{otherwise,}
  \end{cases}
\]
and $\mathrm{consequence}(y_{\mathrm{rec}})\in[0,1]$ checks whether the recovery-turn output contains the held call, matched by function name with lenient argument comparison. The $\mathrm{no\_write}$ gate permits deferral and read-only queries (\S\ref{sec:caseA-setup}); $\mathrm{consequence}(y_{\mathrm{rec}})$ supplies the reward-only recovery signal.

\paragraph{Approximation in the name-based no-write gate.} The implemented name-based gate approximates ``no state-changing write'' over the tool set without verifying irreversibility, so it is not an exact read/write partition; memory uses an additive alternative (App.~\ref{sec:caseB}).

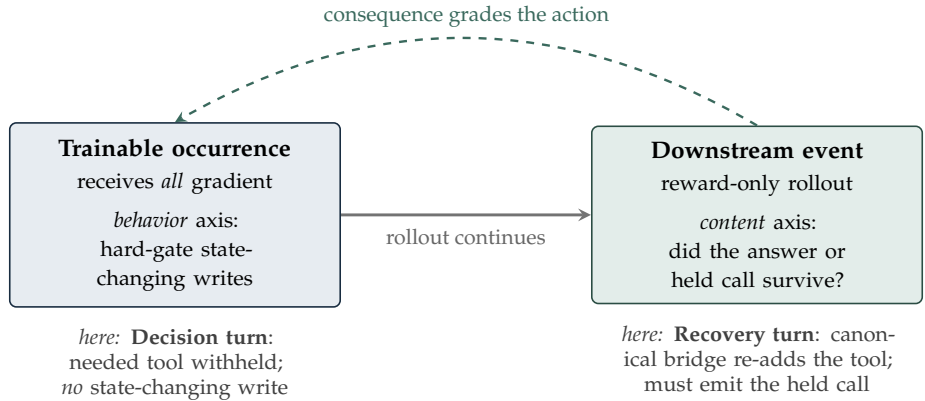
\begin{figure}[!htbp]\centering
\begin{tikzpicture}[
  font=\footnotesize,
  box/.style={rounded corners=2.5pt, draw=black!45, line width=0.6pt,
              text width=3.95cm, minimum height=2.15cm, inner sep=6pt, align=center},
  sub/.style={font=\scriptsize, align=center, text width=4.2cm, text=black!75},
  arr/.style={->,>=stealth, line width=1.1pt, draw=black!55}]
  \node[box, fill=tolblue!12, draw=tolblue!45!black] (cs) at (0,0)
    {{\bfseries Trainable occurrence}\\[1pt]{\scriptsize receives \emph{all} gradient}\\[5pt]
      {\scriptsize \emph{behavior} axis:\\ hard-gate state-changing writes}};
  \node[box, fill=tolteal!16, draw=tolteal!55!black] (ds) at (7.7,0)
    {{\bfseries Downstream event}\\[1pt]{\scriptsize reward-only rollout}\\[5pt]
     {\scriptsize \emph{content} axis:\\ did the answer or held call survive?}};
  \draw[arr] (cs.east) -- node[below=1pt,font=\scriptsize,text=black!60]{rollout continues} (ds.west);
  \draw[->,>=stealth,dashed,line width=0.9pt,draw=tolteal!75!black]
    (ds.north) to[bend right=30]
    node[above,font=\scriptsize,text=tolteal!75!black]{consequence grades the action} (cs.north);
  \node[sub, anchor=north] at ([yshift=-4pt]cs.south)
    {\emph{here:} \textbf{Decision turn}: needed tool withheld; \emph{no} state-changing write};
  \node[sub, anchor=north] at ([yshift=-4pt]ds.south)
    {\emph{here:} \textbf{Recovery turn}: canonical bridge re-adds the tool; must emit the held call};
\end{tikzpicture}
\caption{Action/consequence decoupling in the Nemotron configuration: the decision-turn occurrence receives gradient, while reward-only recovery supplies its consequence signal.}
\label{fig:decision-turn}
\floatnote{This figure instantiates Eq.~\eqref{eq:m2-mult}; category-specific training and the precomputed keyword proxy appear in App.~\ref{sec:caseA-nothink} and App.~\ref{sec:caseB}.}
\end{figure}
}

\section{Cross-setting applications}\label{sec:applications}

Beyond the BFCL four-cell study (\S\ref{sec:caseA-predictor}), the applications below retain the occurrence-local training unit and three diagnostic gates (\S\ref{sec:recipe}), while adapting the training location and label. They illustrate three label sources: the current action, sampled recovery, and a precomputed proxy.

\subsection{Logged interaction traces: repeat-call avoidance}

In logs of real interactions, Nemotron-Super-120B repeats a prior non-error $(\text{name},\text{arguments})$ call after a usable result on $30\%$ of flagged steps. On the same traces, \mbox{GPT-4.1}~\cite{openai2025gpt41} has a measured repeat rate of $0.61\%$, leaving substantial reference-policy headroom. Training updates the next action at each flagged trace position using the deterministic $\mathrm{repeat\_of\_history}$ label. Given the logged context, this label depends only on the current action, so its variation is action-dependent and has no downstream sampling term. \methodname{} raises call-versus-answer agreement with \mbox{GPT-4.1} from $37\%$ to about $75\%$ across three training seeds ($72$--$78\%$). The note to Table~\ref{app:web-repeat} gives additional evaluation details.

\subsection{Nemotron: missing-function application}\label{sec:caseA-m1m2}\label{sec:caseA-results}

In this application, a needed tool is withheld and reintroduced on the next turn. The decision occurrence receives all gradient; a no-write gate scores its action, and the consequence in Eq.~\eqref{eq:m2-mult} comes from sampled reward-only recovery (\S\ref{sec:recipe:m1}; App.~\ref{sec:recipe:formal}). \methodname{} yields a paired \missfunc{} gain of $+4.4$pt over the Nemotron start under repeated deterministic evaluation. Gemma instead selects the \missfunc{} recovery occurrence (App.~\ref{sec:caseA-nothink}); App.~\ref{sec:caseA-details} gives the Nemotron instantiation and evaluation.

\subsection{Memory management: xLAM and Gemma applications}

In the xLAM setting, a requested add fails when memory is full. Metadata selects the storage response after that failure for training; the response must remove a safe entry before adding the new one. An additive behavior--content reward scores operation order and keyword retention, using a precomputed proxy and hard gates for invalid actions. The BFCL v4 agentic/memory sub-task mean moves from $34.54\%$ for the xLAM start to $50.54\%$ for the selected memory-trained checkpoint (App.~\ref{sec:caseB}).

The no-think Gemma-4-26B-A4B application extends the memory setting to both storage and retrieval occurrences. A retuned training mixture uses separate single-occurrence examples for storage and retrieval: storage uses the additive behavior--content reward and precomputed keyword proxy, while retrieval is scored by coverage of facts needed for the expected query. Its memory sub-task mean moves from $38.1\%$ to $52.7\%$ (App.~\ref{sec:caseB-gemma}).

\long\gdef\memorymethoddetails{%
\subsubsection{xLAM memory application}\label{sec:caseB-setup}

\paragraph{Scenario and selected occurrence.}\label{sec:caseB-m1} When memory is full, an attempted \texttt{core\_memory\_add} fails. The agent must delete a safely removable entry before adding the new one, rather than repeat the failed add or use \texttt{core\_memory\_clear}/\texttt{archival\_memory\_clear}, a \emph{destructive clear} that erases downstream information. Using the xLAM start on BFCL v4 agentic/memory, occurrence-local RL places gradient only on that memory-storage response. Candidate membership and the keys that must be preserved or may be removed are supplied by metadata; the surrounding history is fixed context rather than part of the gradient-bearing trajectory.

\paragraph{Action/consequence example: negation loss.} Figure~\ref{fig:case-b-case69} traces one observed storage-to-retrieval failure (record 69).

\begin{figure}[H]
\centering
\begin{tikzpicture}[
  stage/.style={
    rounded corners=2.5pt,
    line width=0.55pt,
    text width=0.42\linewidth,
    align=left,
    inner xsep=5pt,
    inner ysep=3pt,
    font=\scriptsize
  },
  causal arrow/.style={
    -{Stealth[length=2.2mm,width=1.6mm]},
    draw=tolred!75,
    line width=0.9pt
  },
  failure note/.style={
    rounded corners=2pt,
    draw=tolred!38,
    fill=tolred!5,
    text width=0.92\linewidth,
    align=left,
    inner xsep=5pt,
    inner ysep=3pt,
    font=\scriptsize
  }
]
\node[stage, fill=tolcyan!8, draw=tolblue!38] (storage) {%
  {\footnotesize\bfseries\color{tolblue}STORAGE TIME}\hfill
  {\scriptsize\color{tolgrey}candidate occurrence}\\[1pt]
  {\scriptsize\color{tolgrey}Trigger: memory-full write failure}\\[1pt]
  \textcolor{tolblue}{\texttt{[user]}}\;
  ``\textit{success in finance \textbf{means nothing} if you sacrifice the relationships that matter most.}''\\[2pt]
  \textcolor{tolpurple}{\texttt{[store]}}\;
  \texttt{core\_memory\_add(\ldots)} $\to$
  \texttt{"values relationships over financial success"}\\[1pt]
  {\color{tolred}\bfseries Absolute negation dropped}
};
\node[stage, fill=tolsand!12, draw=tolyellow!45,
      right=0.025\linewidth of storage] (retrieval) {%
  {\footnotesize\bfseries\color{tolyellow}RETRIEVAL TIME}\hfill
  {\scriptsize\color{tolgrey}two turns later}\\[1pt]
  {\scriptsize\color{tolgrey}Downstream consequence}\\[1pt]
  \textcolor{tolblue}{\texttt{[user]}}\;
  ``\textit{how much does success in finance mean if I sacrifice the relationships that matter the most?}''\\[2pt]
  \textcolor{tolpurple}{\texttt{[ans]}}\;
  correct: \texttt{Nothing}\\[-1pt]
  \hphantom{\texttt{[ans]}\;}model:
  ``\textit{It means very little, if anything\,\ldots}''\\[1pt]
  {\color{tolred}\bfseries Wrong consequence}
};
\draw[causal arrow] (storage.east) -- (retrieval.west);
\coordinate (stagebottom) at ($(storage.south)!0.5!(retrieval.south)$);
\node[failure note, anchor=north] at ($(stagebottom)+(0,-1.5mm)$) {%
  \textbf{Failure.} Storage omits ``means nothing''; the later answer weakens the original absolute negation to ``very little.'' This trace illustrates the action/consequence gap in Fig.~\ref{fig:decision-turn}.
};
\end{tikzpicture}
\caption{Action/consequence trace for a memory-storage occurrence.}
\label{fig:case-b-case69}
\floatnote{The precomputed keyword proxy comes from the same record 69. This trace illustrates the action/consequence split in Fig.~\ref{fig:decision-turn}.}
\end{figure}
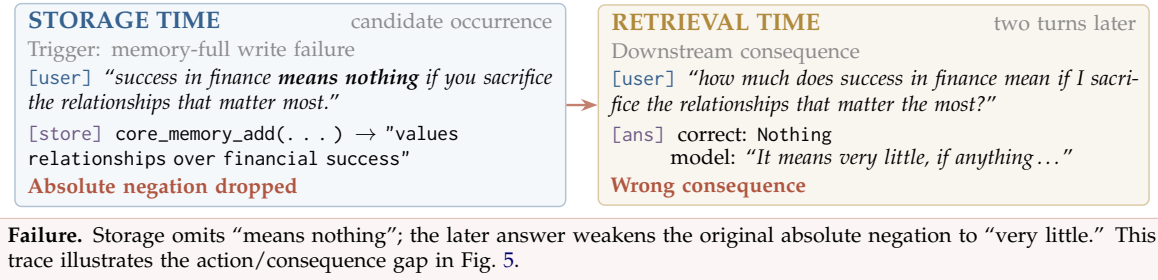

\paragraph{Occurrence-local label.}\label{sec:caseB-m2}\label{sec:caseB-m3} The reward combines correct-sequence behavior with a precomputed GPT-4o content proxy: 2--5 required terms, including meaningful negations, matched against storage text. Training scores this content locally; BFCL evaluates the full store$\to$retrieve$\to$answer chain.

\memoryrewardformuladetails

}

\long\gdef\memoryxlamtable{%
\begin{table}[H]
\centering
\caption{BFCL v4 agentic/memory results (\%): start $\to$ memory-trained checkpoint.}
\label{tab:caseB}
\label{tab:caseB-gemma}
\setlength{\tabcolsep}{4pt}
\renewcommand{\arraystretch}{1.15}
\begin{tabular}{@{}>{\raggedright\arraybackslash}p{0.25\linewidth}>{\centering\arraybackslash}p{0.15\linewidth}>{\centering\arraybackslash}p{0.15\linewidth}>{\centering\arraybackslash}p{0.16\linewidth}>{\centering\arraybackslash}p{0.15\linewidth}@{}}
\toprule
& \multicolumn{3}{c}{memory sub-tasks} & {} \\
\cmidrule(lr){2-4}
\textbf{Application} & \textbf{Key--value store} & \textbf{Vector store} & \textbf{Recursive summary} & \textbf{Sub-task mean} \\
\midrule
xLAM & $11.48\to41.94$ & $43.35\to45.16$ & $48.77\to64.52$ & $34.54\to50.54$ \\
Gemma-4-26B-A4B & $36.8\to47.1$ & $28.4\to56.8$ & $49.0\to54.2$ & $38.1\to52.7$ \\
\bottomrule
\end{tabular}
\floatnote{The sub-task mean is unweighted and is not BFCL weighted full-suite Overall. The xLAM starting values average 10 memory-only evaluation runs, with population standard deviation 1.93 for the sub-task mean; trained values are single evaluations of selected checkpoints. The Gemma trained result uses step~20 from the highest-memory mixture among five data-mixture runs.}
\end{table}
}

\section{Discussion}
\label{sec:cross}
\label{sec:discussion}

\paragraph{Why diagnosis precedes training.}\label{sec:analysis-summary}
The trainable call depends on the task and model configuration. The four-cell study selects different turns across categories, while the Gemma pilot shows that \missfunc{} can shift from decision to recovery across model-and-reasoning configurations (App.~\ref{sec:caseA-nothink}). Applying the method to a new setting therefore means reassessing the candidate calls, not reusing a fixed training turn.

\long\gdef\crosscasedetails{%
\begin{table}[H]
\centering
\small
\captionsetup{font=small,skip=2pt}
\setlength{\tabcolsep}{3pt}
\setlength{\defaultaddspace}{8pt}
\renewcommand{\arraystretch}{1.10}
\caption{Four Critical-State RL application configurations from \S\ref{sec:applications}.}
\label{tab:instances}\label{app:web-repeat}
\begin{tabular}{@{}>{\raggedright\arraybackslash}p{2.20cm}>{\raggedright\arraybackslash}p{3.10cm}>{\raggedright\arraybackslash}p{5.20cm}>{\raggedright\arraybackslash}p{3.10cm}@{}}
\toprule
\textbf{Application} & \textbf{Selected occurrence} & \textbf{Label source and construction} & \textbf{Reported result} \\
\midrule
Logged repeat-call (Nemotron) & next response at a flagged trace position & deterministic action-only \texttt{repeat\_of\_history} & call/answer agreement with GPT-4.1: $37\%$ to about $75\%$ ($72$--$78\%$) \\
\addlinespace
Missing function (Nemotron) & decision before tool availability & reward-only sampled recovery; no-write gate; multiplicative, Eq.~\eqref{eq:m2-mult} & BFCL \missfunc{} gain: $+4.4$pt \\
\addlinespace
Memory (xLAM) & memory-full storage response & precomputed GPT-4o keyword proxy; additive behavior--content reward with hard gates & sub-task mean $34.54\%\to50.54\%$ \\
\addlinespace
Memory (Gemma) & storage or retrieval response & storage: keyword proxy; retrieval: expected-query fact coverage & sub-task mean $38.1\%\to52.7\%$ \\
\bottomrule
\end{tabular}
\floatnote{The shared training unit is one selected response; its occurrence-local label is task-specific (\S\ref{sec:recipe:m2}). \textit{Repeat-call.} The agreement range spans three training seeds. \applicabilityillustration}
\end{table}
}

\long\gdef\crosscaseprosedetails{%
The theory explains three aspects of this workflow. Action-dependent variance measures unrestricted local-label improvement potential; finite-sample analysis bounds ranking error (App.~\ref{app:theory}). Local credit removes noise from independent sibling outcomes that enters trajectory credit (App.~\ref{sec:recurrent}). Updating shared parameters can also change other calls; the drift analysis bounds the expected-return gap between using the updated policy only at the selected call and using it throughout the rollout (App.~\ref{sec:stopping}).
}

\crosscaseprosedetails

\section{Limitations}
The four-cell study tests categorical turn selection rather than calibrating variance magnitudes or diagnostic thresholds. Repeat-call, Nemotron, and memory are applications rather than matched tests of the diagnostic or design choices. The four-cell, Nemotron missing-function, and memory settings have $K{=}1$. Logged repeat-call training uses single-step examples from flagged trace positions. The recurrence and drift results are conditional (Apps.~\ref{sec:recurrent} and~\ref{sec:stopping}). App.~\ref{app:prov} reports evaluation repetition, serving variability, and training-unit provenance.

\long\gdef\applicabilityillustration{%
Training uses fixed logged contexts without online exploration. On the traces summarized in \S\ref{sec:applications}, the trained models skip $2$--$3$ of $20$ needed calls across seeds, compared with $2$ at the start.
}

\long\gdef\conclusionscopedetails{%
\methodname{} diagnoses action-dependent learning signal and uses it to select model calls for local training. The BFCL four-cell study on no-think Gemma-4-26B-A4B supports the assignments \texttt{miss\_func} $\to$ recovery and \texttt{miss\_param} $\to$ decision: selected turns improve, while alternatives stay flat or decline (\S\ref{sec:caseA-predictor}). Across applications, the training location and reward adapt to the task, while the workflow remains the same: diagnose where to train, then optimize the selected call.
}

\section{Conclusion}
\label{sec:conclusion}

\conclusionscopedetails

\clearpage
\footnotesize
\phantomsection

\normalsize
\appendix

\medskip\noindent\textbf{Appendix roadmap.} Appendices~\ref{app:notation} and~\ref{app:theory} provide notation and core analysis. Appendix~\ref{app:implementation} covers the recipe (\ref{app:recipe}), candidate identification (\ref{app:localize}), and evaluation (\ref{app:prov}); \ref{app:applications} presents applications (overview: Table~\ref{app:web-repeat}; Nemotron: \ref{sec:caseA-details}); \ref{app:extended} analyzes credit and drift (\ref{sec:recurrent}--\ref{sec:stopping}); \ref{app:prompts} provides prompts.

\section{Notation}
\label{app:notation}
The method's phase, occurrence, and selected occurrence are defined in \S\ref{sec:recipe}; Def.~\ref{def:cs} formalizes ``critical state.'' The table collects notation for the general formalism, recurrence analysis, and stopping analysis. ``Local-label credit'' denotes gradient estimation using an occurrence-local label.

\begingroup
\footnotesize
\renewcommand{\arraystretch}{1.03}
\makeatletter
\patchcmd{\LT@output}{\copy\LT@foot\vss}{\copy\LT@foot\vfil}{}{}
\patchcmd{\LT@output}{\copy\LT@foot\vss}{\copy\LT@foot\vfil}{}{}
\makeatother
\setcounter{notationtablecounter}{\value{table}}
\begin{longtable}{@{}>{\raggedright\arraybackslash}p{\dimexpr 0.20\linewidth+\tabcolsep+\tabcolsep\relax}>{\raggedright\arraybackslash}p{\dimexpr 0.74\linewidth-\tabcolsep-\tabcolsep\relax}@{}}
\toprule
\textbf{Term or symbol} & \textbf{Meaning and role} \\
\midrule
\endfirsthead
\multicolumn{2}{@{}l}{\textit{Notation (continued)}}\\
\addlinespace
\toprule
\textbf{Term or symbol} & \textbf{Meaning and role} \\
\midrule
\endhead
\multicolumn{2}{@{}l}{\textbf{General critical-state formalism (App.~\ref{app:theory})}}\\*
$z_i,\ Q(x,a)$ & occurrence-local label and its mean $Q(x,a)=\mathbb E[z\mid x,a]$, distinct from App.~\ref{sec:stopping}'s anchor \emph{return} value $Q_{\theta_0}$.\\
$V_{\mathrm{act}},\ V_{\mathrm{acc}},\ F$ & at a fixed prefix under the current policy: action-dependent label variance $\operatorname{Var}_aQ$; its score-accessible component $\mu^\top F^\dagger\mu$; and score second moment $F=\mathbb E[gg^\top]$. Here $\mu=\mathbb E[g(z-q)]$ and $F^\dagger$ is the Moore--Penrose inverse.\\
$\varphi_x(z),\ \lambda(x)$ & $\varphi_x(z)=\mathbb E[R\mid x,z]$ (assumed affine); local slope $\lambda(x)=\partial_z\varphi_x$ relating the local-label and trajectory gradients.\\
$q(x),\ \mathcal F_i$ & \textbf{Label baseline.} Context-conditioned $q(x)=\mathbb E[z\mid x]$ used by the local-label estimator. \textbf{Conditioning field.} $\mathcal F_i=\sigma(x_i,a_i,z_i)$, where the operator denotes the generated sigma-field, not the critical-state symbol below.\\
$\pi_0,\ \Delta_{\mathrm{head}},\ h$ & reference policy for headroom; required margin $\Delta_{\mathrm{head}}>0$; occurrence-local label window ($\le h$ turns, $h\ll$ episode), Def.~\ref{def:cs}.\\
$w$ & return residual $R-z$: return variation outside the occurrence-local label in general, $\sum_{j\ne i}z_j$ under $K$-fold recurrence (Cor.~\ref{cor:rec}).\\
$S_z$ & fixed-prompt label share $\operatorname{Var}(z)/\operatorname{Var}(R)$; the advantage-multiplier SNR ratio is $\sqrt{S_z}$ under additive independence (Cor.~\ref{cor:rec}).\\
\addlinespace
\multicolumn{2}{@{}l}{\textbf{Nemotron missing-function reward (\S\ref{sec:recipe:m2})}}\\*
$z_{\mathrm{tool}}=\mathrm{no\_write}(a_{\mathrm{dec}})\times\mathrm{consequence}(y_{\mathrm{rec}})$ & multiplicative decision-turn occurrence-local label $z$ of Thm~\ref{thm:rb}.\\*
$\mathrm{no\_write}(a_{\mathrm{dec}})\in\{0,1\}$ & process gate: $1$ iff the decision-turn action emits no state-changing write.\\*
$\mathrm{consequence}(y_{\mathrm{rec}})\in[0,1]$ & score for whether the recovery-turn output contains the required tool call.\\*
$a_{\mathrm{dec}}$ / $y_{\mathrm{rec}}$ & decision-turn action and recovery-turn output, respectively.\\
\addlinespace
\multicolumn{2}{@{}l}{\textbf{Recurrence-variance symbols (App.~\ref{sec:recurrent})}}\\*
$K$ & retained visits to the abstract critical state $\sigma$ within one episode.\\
$\sigma$ & abstract critical state: an equivalence class of qualifying contexts, not a fixed turn.\\
$a_i,\ c_i$ & action taken and correct action at occurrence $i$, respectively; the required action may vary across occurrences.\\
$r_i=\mathbf 1[a_i=c_i]$ & local correctness at occurrence $i$.\\
$q_i,\ s_i^2$ & Bernoulli success probability $q_i=\Pr_\theta(a_i=c_i)$, the label baseline for local correctness; $s_i^2=q_i(1-q_i)=\mathrm{Var}(r_i)$.\\
$R=\sum_i r_i$ & terminal (trajectory-level) reward.\\
$b,\ A_t$ & trajectory-return baseline and turn $t$ advantage. In the recurrence analysis, the baseline is the population analogue of GRPO's group mean; the general formalism conditions that baseline on context.\\
$\hat g^{\,\mathrm{flat}}$ & single-occurrence local-label gradient estimator ($O(1)$ variance under the stylized recurrence assumptions); \texttt{traj} denotes the trajectory-return estimator.\\
$v_{-i},\ v_i,\ f_i$ & sibling-label variance $\sum_{j\ne i}s_j^2$; local gradient trace variance; score second moment $\mathbb E\|g_i\|^2$. Shared credit adds $v_{-i}f_i$ to the per-occurrence trace variance.\\
$\mathrm{SNR}\propto 1/\sqrt K$ & per-occurrence gradient SNR scaling under the nondegenerate, independent equal-variance model; the scalar advantage-multiplier SNR ratio is exactly $1/\sqrt K$.\\
\addlinespace
\multicolumn{2}{@{}l}{\textbf{Drift-bias and stopping symbols (App.~\ref{sec:stopping})}}\\*
$J(\theta)=\mathbb E_\tau[R(\tau)]$ & true multi-turn objective.\\
$P_\theta^{\mathrm{hyb}},\ J_{\mathrm{hyb}}$ & hybrid trajectory law using $\pi_\theta$ at the selected call and $\pi_{\theta_0}$ elsewhere; its expected return is the exact frozen surrogate.\\
$d_0,\ Q_{\theta_0},\ \widehat Q$ & anchor selected-context distribution, anchor continuation return value, and its frozen proxy. $\widehat J=\mathbb E_{d_0,\pi_\theta}\widehat Q$.\\
$D_{\mathrm{out}}$ & $\mathrm{KL}(P_\theta\|P_\theta^{\mathrm{hyb}})$: summed KL outside the selected occurrence, averaged over current-policy rollouts.\\
$D_{\mathrm{sel}}^0$ & selected-action KL averaged over anchor contexts, equal to $\mathrm{KL}(P_\theta^{\mathrm{hyb}}\|P_{\theta_0})$.\\
$\varepsilon_0$ & uniform upper bound on return-value proxy error; its contribution to the gain bound vanishes at the anchor (Eq.~\eqref{eq:proxy-drift-gain}).\\
$\delta,\ p(\delta),\ G(\delta)$ & $\delta=\sqrt{D_{\mathrm{out}}}$, exact surrogate gain $p=\Delta J_{\mathrm{hyb}}$, and gain lower bound $G=p-\kappa\delta$, with $\kappa=\sqrt{2}R_{\max}$; $\delta^\star$ maximizes $G$ (Cor.~\ref{cor:stop}).\\
\bottomrule
\end{longtable}
\setcounter{table}{\value{notationtablecounter}}
\normalsize
\endgroup

\section{Core analysis}\label{app:theory}
\subsection{When a candidate phase is a critical state}
\label{sec:recipe:formal}
Fix a policy $\pi_\theta$ and terminal return $R(\tau)\in[-R_{\max},R_{\max}]$. A phase $\sigma$ is an equivalence class of contexts, such as missing-function defer/recover, rather than a single fixed context. After routing, an episode has $K\ge1$ retained occurrences, each with context $x_i$, action $a_i$, and score $g_i=\nabla_\theta\log\pi_\theta(a_i\mid x_i)$. Its occurrence-local segment includes a bounded reward-only continuation.

\begin{definition}[Critical state]\label{def:cs}
Relative to a policy $\pi_\theta$ and a reference $\pi_0$, $\sigma$ is a \emph{critical state} if it admits an occurrence-local label $z_i$, supplied by the environment or an oracle/proxy and measurable within at most $h$ turns ($h\ll$ the episode length). With $Q(x,a):=\mathbb E[z\mid x,a]$, every retained context satisfies
\[
\underbrace{R \perp a_i \mid (x_i,z_i)}_{\text{(i) action-sufficiency}},\qquad
\underbrace{\max_a Q(x,a)-\mathbb E_{a\sim\pi_0}Q(x,a)\ge \Delta_{\mathrm{head}}>0}_{\text{(ii) headroom}},\qquad
\underbrace{\operatorname{Var}_{a\sim\pi_\theta(\cdot\mid x)}Q(x,a)>0}_{\text{(iii) trainability}} .
\]
\end{definition}

The headroom margin is task-specific. Contexts failing either (ii) or (iii) are split or filtered before ranking. The bounded label window prevents the vacuous choice $z:=R$ for any state. Benchmark-directed training also requires a return-aligned label; Theorem~\ref{thm:rb} gives the positive-slope condition for per-prompt gradient alignment. For no-think Gemma, $\mathrm{no\_write}$ near $0.97$ leaves little headroom in this gate; the pilot's training comparison favors recovery (App.~\ref{sec:caseA-nothink}).

For the missing-function decision, the full label is $z=\mathrm{no\_write}(a_{\mathrm{dec}})\times\mathrm{consequence}(y_{\mathrm{rec}})$, denoted $z_{\mathrm{tool}}$ in Eq.~\eqref{eq:m2-mult}; $R$ remains the benchmark terminal return. This operational proxy checks the held call and the decision's process gate; the benchmark also checks extraneous recovery calls. Exact sufficiency is the theoretical condition on the whole label. Sampled recovery makes $z$ stochastic even after the decision action is fixed.

\subsection{From the diagnostic to local improvement}
\paragraph{Action-dependent variance measures local improvement potential.}
Fix a prefix $x$ and freeze the continuation kernel that supplies $z$. Write $p(a)=\pi_\theta(a\mid x)$, $Q(a)=\mathbb E[z\mid x,a]$, $q=\mathbb E_p Q$, and $V_{\mathrm{act}}=\operatorname{Var}_p Q$. The score is $g(a)=\nabla_\theta\log p(a)$, with $\mathbb E_p g=0$; vector variance denotes trace covariance.

\begin{proposition}[Continuation noise and the occurrence gradient]\label{prop:diagnostic:noise}
For a fixed baseline $b=b(x)$ and finite second moments,
\begin{align*}
\mathbb E[g(z-b)]&=\mathbb E_p[g(Q-q)]=:\mu,\\
\operatorname{Var}\!\big(g(z-b)\big)
&=\operatorname{Var}_p\!\big(g(Q-b)\big)
+\mathbb E_p\!\left[\|g\|^2\operatorname{Var}(z\mid x,a)\right].
\end{align*}
Thus $V_{\mathrm{act}}=0$ implies $\mu=0$, even when sampled labels vary.
\end{proposition}
\begin{proof}
Condition on $a$ to get mean $g(Q-b)$ and variance $\|g\|^2\operatorname{Var}(z\mid x,a)$. Total expectation and variance, with $\mathbb E_p g=0$, give the identities.
\end{proof}

\par\smallskip\noindent
\begin{minipage}{\linewidth}
\centering
\begin{tikzpicture}[font=\small,y=0.85cm,
  mean cell/.style={draw=tolgrey!60,minimum width=0.68cm,minimum height=0.43cm,inner sep=0pt}]
\node[font=\small\bfseries] at (2.9,1.20) {Action-conditioned means $Q(a)$};
\foreach \x/\i in {1.4/1,2.4/2,3.4/3,4.4/4}
  \node at (\x,0.80) {$a_{\i}$};
\node at (5.8,0.80) {$V_{\mathrm{act}}$};
\node[anchor=east] at (0.75,0.37) {Candidate 1};
\node[anchor=east] at (0.75,-0.16) {Candidate 2};
\foreach \x/\shade/\value in {1.4/0/0,2.4/0/0,3.4/28/1,4.4/28/1}
  \node[mean cell,fill=tolblue!\shade] at (\x,0.37) {$\value$};
\foreach \x/\shade/\value in {1.4/0/0,2.4/14/{\tfrac12},3.4/14/{\tfrac12},4.4/28/1}
  \node[mean cell,fill=tolblue!\shade] at (\x,-0.16) {$\value$};
\node[text=tolblue] at (5.8,0.37) {$1/4$};
\node[text=tolblue] at (5.8,-0.16) {$1/8$};
\draw[tolgrey] (6.35,0.37) -- (7.15,0.10);
\draw[tolgrey] (6.35,-0.16) -- (7.15,0.10);
\draw[-{Stealth[length=5pt]},tolgrey] (7.15,0.10) -- (8.45,0.10);
\node[font=\footnotesize,align=center,text=tolgrey] at (7.50,0.72) {Average\\over actions};
\node[font=\small\bfseries] at (10.2,1.20) {Same label distribution};
\draw[draw=tolgrey,fill=tolgrey!25] (9.25,-0.20) rectangle (9.95,0.38);
\draw[draw=tolgrey,fill=tolgrey!25] (10.45,-0.20) rectangle (11.15,0.38);
\node at (9.60,0.68) {$1/2$};
\node at (10.80,0.68) {$1/2$};
\node[below=2pt] at (9.60,-0.20) {$z=0$};
\node[below=2pt] at (10.80,-0.20) {$z=1$};
\end{tikzpicture}
\captionsetup{font=footnotesize}
\captionof{figure}{\textbf{Same marginal labels, different action signal.} With uniform base and reference policies, $(0,0,1,1)$ and $(0,\tfrac12,\tfrac12,1)$ are conditional Bernoulli means. Both give $\operatorname{Bernoulli}(1/2)$ labels and $1/2$ headroom, but $V_{\mathrm{act}}=1/4$ and $1/8$.}
\label{fig:diagnostic-marginals}
\end{minipage}\par

No selector using only independent marginal labels and headrooms can distinguish swapped candidates, whose rankings reverse. Its worst-case error is at least $1/2$, at any sample size.

The small-KL sensitivity expansion~\cite{lam2016sensitivity} gives the optimization meaning below; its exponential-tilt optimizer also appears in relative-entropy policy search~\cite{peters2010reps}.
\begin{theorem}[Local improvement at a fixed KL budget]\label{thm:diagnostic:gain}
Let $p$ have positive mass on a finite action set, and allow any new distribution $p'$ on that same set. Holding $Q$ fixed, define
\[
\mathcal G_x(\epsilon):=\max_{D_{\mathrm{KL}}(p'\|p)\le\epsilon}
\big(\mathbb E_{p'}Q-\mathbb E_pQ\big).
\]
If $V_{\mathrm{act}}>0$, then, as $\epsilon\downarrow0$,
\[
\mathcal G_x(\epsilon)=\sqrt{2\epsilon V_{\mathrm{act}}}+O(\epsilon).
\]
If $V_{\mathrm{act}}=0$, then $\mathcal G_x(\epsilon)=0$ for every $\epsilon$.
\end{theorem}
\begin{proof}
Set $A=Q-q$ and $\psi(\eta)=\log\mathbb E_p e^{\eta A}$. Exponential tilting gives $p_\eta(a)=p(a)e^{\eta A(a)-\psi(\eta)}$, with gain $\psi'(\eta)$ and KL $\eta\psi'(\eta)-\psi(\eta)$. For any $p'$,
\[
D_{\mathrm{KL}}(p'\|p_\eta)
=D_{\mathrm{KL}}(p'\|p)-\eta\mathbb E_{p'}A+\psi(\eta)\ge0.
\]
Consequently, $p_\eta$ maximizes gain at its own KL budget. Since $\psi(\eta)=V_{\mathrm{act}}\eta^2/2+O(\eta^3)$, choosing that budget to equal $\epsilon$ yields $\eta=\sqrt{2\epsilon/V_{\mathrm{act}}}+O(\epsilon)$ and the stated expansion. When $V_{\mathrm{act}}=0$, $Q$ is constant on the action set.
\end{proof}

\paragraph{Average improvement over model-generated prefixes.}
Fix a finite-support prefix distribution $\rho$, set $\bar V_{\mathrm{act}}=\mathbb E_\rho V_{\mathrm{act}}(x)$, and assume Theorem~\ref{thm:diagnostic:gain}'s conditions at each prefix. With continuation values fixed, the largest average local-label gain over unrestricted conditional policies is
\[
\mathcal G_\rho(\epsilon)=\sqrt{2\epsilon\bar V_{\mathrm{act}}}+O(\epsilon),
\qquad\mathbb E_\rho D_{\mathrm{KL}}(p'_x\|p_x)\le\epsilon.
\]
This holds for $\bar V_{\mathrm{act}}>0$ as $\epsilon\downarrow0$; zero variance gives zero gain at every budget. Normalize the preceding tilt separately at each prefix and set $\Psi(\eta)=\mathbb E_\rho\log\mathbb E_{p_x}e^{\eta[Q(x,a)-q(x)]}$. Averaging the KL identity proves optimality; $\Psi''(0)=\bar V_{\mathrm{act}}$ gives the expansion. Prefix-mean differences supply no gain because $\rho$ is fixed.

\paragraph{The part accessible to a parameterized policy.}
At a fixed prefix, let $F=\mathbb E_p[gg^\top]$, let $F^\dagger$ be its Moore--Penrose inverse, and define $V_{\mathrm{acc}}=\mu^\top F^\dagger\mu$. Compatible-function projection and the quadratic KL model~\cite{kakade2001natural} give
\begin{align*}
V_{\mathrm{act}}&=V_{\mathrm{acc}}+
\mathbb E_p\!\left[(Q-q-g^\top F^\dagger\mu)^2\right],\\
\max_{u^\top Fu/2\le\epsilon}\mu^\top u
&=\sqrt{2\epsilon V_{\mathrm{acc}}}.
\end{align*}
Indeed, $\mu$ belongs to the range of $F$, and $g^\top F^\dagger\mu$ is the least-squares projection of $Q-q$ onto the score span. Orthogonality proves the first identity; Cauchy--Schwarz in the $F$ metric proves the second. A full categorical policy has $V_{\mathrm{acc}}=V_{\mathrm{act}}$.

Under a common label and equal small KL budgets, $V_{\mathrm{act}}$ ranks leading-order unrestricted potential at a fixed prefix. Across a fixed distribution of prefixes, a shared update under an average quadratic KL budget uses $\bar\mu=\mathbb E_x\mu_x$ and $\bar F=\mathbb E_xF_x$, with potential $\sqrt{2\epsilon\bar\mu^\top\bar F^\dagger\bar\mu}$. Applying the same projection under the joint law $\rho(x)p_x(a)$ gives $\bar\mu^\top\bar F^\dagger\bar\mu\le\bar V_{\mathrm{act}}$. Opposing prefix gradients can cancel. The diagnostic is gradient-free; the four-cell intervention tests the shared-policy response.

\subsection{Finite-sample estimation and occurrence selection}
For a balanced nested design at a fixed prefix, draw $n\ge2$ independent actions and $m\ge2$ continuations per action, all independent conditional on the actions. Let $S_{\mathrm{between}}^2$ be the sample variance (denominator $n-1$) of the $n$ action means, and $S_{\mathrm{within}}^2$ the average within-action sample variance (denominator $m-1$). Writing $V_{\mathrm{cont}}=\mathbb E_a\operatorname{Var}(z\mid x,a)$ gives
\[
\mathbb E S_{\mathrm{between}}^2=V_{\mathrm{act}}+\frac{V_{\mathrm{cont}}}{m},
\qquad
\widehat V_{\mathrm{act}}=S_{\mathrm{between}}^2-\frac{S_{\mathrm{within}}^2}{m},
\qquad
\mathbb E\widehat V_{\mathrm{act}}=V_{\mathrm{act}}.
\]
Total variance of an action mean proves the first equality; $\mathbb E S_{\mathrm{within}}^2=V_{\mathrm{cont}}$ gives the correction. This standard nested-simulation correction~\cite{sun2011nested,goda2017conditional} removes continuation noise in expectation; finite estimates can be negative.

\paragraph{From an unbiased estimate to reliable occurrence selection.}
Classical nested-variance analysis separates action coverage from continuation depth~\cite{sun2011nested}. Here that distinction determines how reliably the diagnostic ranks candidate occurrences.

\begin{proposition}[Finite-sample diagnostic precision]\label{prop:diagnostic:precision}
Under the preceding balanced design with bounded labels, let $d(a)=Q(a)-q$ and $\nu(a)=\operatorname{Var}(z\mid x,a)$. Then
\begin{align}
\mathcal E_{n,m}:=\operatorname{Var}(\widehat V_{\mathrm{act}})
&=\frac{\operatorname{Var}_a(d^2)}{n}
+\frac{4\mathbb E_a[d^2\nu]}{nm}
+\frac{2\mathbb E_a[\nu^2]}{nm(m-1)}\nonumber\\
&\quad+\frac{2(V_{\mathrm{act}}+V_{\mathrm{cont}}/m)^2}{n(n-1)}.
\label{eq:diagnostic:precision}
\end{align}
Consider two qualifying candidate occurrences at fixed prefixes under a common label scale, estimated using independent nested batches. If $V_1>V_2$, write $\Delta=V_1-V_2$ and let $\mathcal E_1,\mathcal E_2$ be their estimator variances from \eqref{eq:diagnostic:precision}. Their misranking probability satisfies
\[
\Pr(\widehat V_2\ge\widehat V_1)
\le\frac{\mathcal E_1+\mathcal E_2}{\Delta^2+\mathcal E_1+\mathcal E_2}.
\]
In particular, $\mathcal E_1+\mathcal E_2\le\delta\Delta^2/(1-\delta)$ guarantees correct selection with probability at least $1-\delta$, for $0<\delta<1$.
\end{proposition}
\begin{proof}
Center labels at $q$. For action $i$, let $Y_i=m^{-1}\sum_j(z_{ij}-q)$, let $S_i^2$ be its within-action sample variance, and set
\[
T_i=Y_i^2-S_i^2/m
=\frac{\sum_{j\ne k}(z_{ij}-q)(z_{ik}-q)}{m(m-1)}.
\]
Conditional independence gives $\mathbb E[T_i\mid a_i]=d(a_i)^2$ and $\operatorname{Var}(T_i\mid a_i)=4d(a_i)^2\nu(a_i)/m+2\nu(a_i)^2/[m(m-1)]$. Algebra rewrites the existing estimator as
\[
\widehat V_{\mathrm{act}}=\frac1n\sum_iT_i
-\frac{\sum_{i\ne k}Y_iY_k}{n(n-1)}.
\]
Since the groups are independent and $\mathbb E Y_i=0$, the two terms are uncorrelated; the second has variance $2(\mathbb E Y_i^2)^2/[n(n-1)]$. Total variance for $T_i$ and $\mathbb E Y_i^2=V_{\mathrm{act}}+V_{\mathrm{cont}}/m$ prove \eqref{eq:diagnostic:precision}. Finally, $\widehat V_2-\widehat V_1$ has mean $-\Delta$ and variance $\mathcal E_1+\mathcal E_2$; the one-sided Chebyshev (Cantelli) inequality gives the ranking bound. For more candidates, sum the pairwise bounds against the unique best candidate.
\end{proof}

\paragraph{More continuations cannot replace more actions.}
As $m\to\infty$ at fixed $n$, continuation noise vanishes but the estimator variance tends to $\operatorname{Var}_a(d^2)/n+2V_{\mathrm{act}}^2/[n(n-1)]$, which is positive when $V_{\mathrm{act}}>0$. Conversely, any fixed $m\ge2$ allows consistent estimation as $n$ grows, and the misranking bound vanishes for a fixed positive gap. Thus independent actions can resolve the ranking without precise estimates of every action's label mean.

\paragraph{Averaging estimates across sampled prefixes.}
For $P$ independent prefixes $X_\ell\sim\rho$ with independent balanced $(n,m)$ batches, set $\widehat{\bar V}_{\mathrm{act}}=P^{-1}\sum_{\ell=1}^P\widehat V_{\mathrm{act}}(X_\ell)$. Conditional unbiasedness and total variance give
\[
\mathbb E\widehat{\bar V}_{\mathrm{act}}=\bar V_{\mathrm{act}},\qquad
\operatorname{Var}(\widehat{\bar V}_{\mathrm{act}})
=\frac{\operatorname{Var}_\rho(V_{\mathrm{act}}(X))+\mathbb E_\rho\mathcal E_{n,m}(X)}{P}.
\]
The same misranking bound applies to independent candidate datasets using these averaged targets and variances. More actions or continuations reduce within-prefix error; more independent prefixes also reduce between-prefix error.

\subsection{From local-label credit to terminal-return credit}
Proposition~\ref{prop:diagnostic:noise} isolates noise \emph{within} the local label. The next result concerns a different source: terminal-return variation \emph{outside} that label.

\begin{theorem}[Per-prompt conditioning on an occurrence-local label]\label{thm:rb}
Let $\hat g^{\mathrm{traj}}_i=g_i(R-b(x_i))$, with $b(x_i)=\mathbb E[R\mid x_i]$, and $\hat g^{\mathrm{flat}}_i=g_i(z_i-q(x_i))$, with $q(x_i)=\mathbb E[z\mid x_i]$. Set $\mathcal F_i=\sigma(x_i,a_i,z_i)$. If $R\perp a_i\mid(x_i,z_i)$ and $\varphi_x(z):=\mathbb E[R\mid x,z]$ is affine with slope $\lambda(x)$, then
\[
\mathbb E[\hat g^{\mathrm{traj}}_i\mid\mathcal F_i]=\lambda(x_i)\hat g^{\mathrm{flat}}_i,
\qquad
\operatorname{Var}(\hat g^{\mathrm{traj}})
=\operatorname{Var}(\lambda(x)\hat g^{\mathrm{flat}})
+\mathbb E[\|g\|^2\operatorname{Var}(R\mid x,a,z)].
\]
\end{theorem}
\begin{proof}
Sufficiency gives $\mathbb E[R\mid x,a,z]=\varphi_x(z)$. Affinity gives $\varphi_x(z)-\mathbb E[R\mid x]=\lambda(x)(z-q(x))$. Multiply by the measurable score and apply total variance.
\end{proof}

At a fixed prompt, $\lambda(x)>0$ preserves gradient direction; across prompts the return gradient weights local gradients by $\lambda(x)$. Binary labels make the link affine automatically; continuous missing-function and memory scores require calibration of the full label. A small $\lambda$ attenuates local credit and a negative one reverses it. If the conditional-mean sufficiency residual is bounded by $\varepsilon$, the conditional-gradient identity has error at most $\varepsilon\|g\|$.

\begin{corollary}[Return residual and independent recurrence]\label{cor:rec}
At a fixed prompt, suppose $R=z+w$ with $w$ independent of $(g,z)$. Both estimators have mean $\mu=\mathbb E[g(z-\mathbb E z)]$, and
\[
\operatorname{Var}(\hat g^{\mathrm{traj}})
=\operatorname{Var}(\hat g^{\mathrm{flat}})+\mathbb E\|g\|^2\operatorname{Var}(w),
\qquad
S_z=\frac{\operatorname{Var}(z)}{\operatorname{Var}(R)}.
\]
For nonzero $\mu$, define the multiplier signal-to-noise ratio (SNR) as $\|\mu\|$ divided by the multiplier's standard deviation. Then $\mathrm{SNR}_{\mathrm{traj}}/\mathrm{SNR}_{\mathrm{flat}}=\sqrt{S_z}$. If $w$ sums $K-1$ independent sibling labels, each with the same positive variance as $z$ and unaffected by the selected action, then $S_z=1/K$. App.~\ref{sec:recurrent} gives the full per-occurrence covariance and sample-cost result. For $K=1$, $w$ can still contain terminal-return variation beyond the local label.
\end{corollary}
\begin{proof}
The centered residual $g(w-\mathbb E w)$ has zero mean and zero covariance with $g(z-\mathbb E z)$. Its trace variance is $\mathbb E\|g\|^2\operatorname{Var}(w)$; scalar variance addition gives $S_z$.
\end{proof}

\paragraph{Empirical single-occurrence scope.}
On Gemma-4-26B-A4B home-buying rollouts ($n{=}6560$), the per-turn-share proxy gives a pooled variance ratio $S_z=\operatorname{Var}(z)/\operatorname{Var}(R)=0.31$ ($95\%$ confidence interval (CI) $[0.29,0.33]$; $\sqrt{S_z}$ is about $0.56$), with slope $\lambda$ near $1.7$ and correlation $0.95$. The field-survey analysis has a signed terminal-reward mean difference of $-0.021$ between correct and wrong loadouts (App.~\ref{sec:rec:evidence}). The ratio and its square root depend on label scale: these are descriptive associations, not measured SNR losses or a fitted $1/\sqrt K$ exponent. The prompt-centered ratio is $\mathbb E[\operatorname{Var}(z\mid x)]/\mathbb E[\operatorname{Var}(R\mid x)]$, equal to the marginal ratio when prompt means do not vary. App.~\ref{sec:recurrent} analyzes group credit; App.~\ref{sec:stopping} analyzes changes to the surrounding policy.

\section{Implementation and evaluation}\label{app:implementation}
\subsection{Applying \methodname{} to a new task}\label{app:recipe}
\criticalstatepractice
\coveragemethoddetails

\subsection{Candidate identification in the evaluated tasks}
\label{app:localize}
\paragraph{Candidate identification.} We proposed a defer-now, act-later candidate phase in \texttt{miss\_func} and \texttt{miss\_param} for occurrence-local RL (\S\ref{sec:recipe:m1}). For these BFCL configurations, we compared failures from the starting checkpoint with those from a stronger reference and grouped the failure patterns by BFCL \texttt{error\_type} (e.g.\ \texttt{empty\_turn\_model\_response} ``defer/refuse'' vs.\ \texttt{instance\_state\_mismatch} ``wrong write''). Human interpretation of these patterns identified the \texttt{empty\_turn} \missfunc{} phase and the refusal pattern in no-think Gemma for further investigation. In memory storage, the memory-full flag directly identifies candidate rows.

\paragraph{Counterfactual check.} We compared the unmodified control with a clean re-ask and an explicit task restatement in place of the canonical bridge. In this one-way clean-prefix counterfactual, the state checker found that a clean prefix eliminated deferral (\texttt{miss\_param} $17\!\to\!0$, \texttt{miss\_func} $10\!\to\!0$), supporting a ``refusal $+$ vague re-ask'' explanation of continued deferral. Eight-bit floating-point (FP8) plus expert-parallel nondeterminism made a full-pass rerun too noisy, so the result remains \emph{directional}.

\paragraph{Row selection and validation.} A deterministic rule identifies candidate decision rows with an empty ground-truth action (for \texttt{miss\_func}, turn $k{-}1$ before the held call, gated on $\mathrm{gt}[k{-}1]\!=\![\,]$ and $\mathrm{gt}[k]$). For Nemotron's turn-0 rows, an LLM judge excludes requests that the available tools can fulfill (YES), retaining only those requiring the withheld tool (App.~\ref{app:prompts}).

\subsection{Training and evaluation details}
\label{app:prov}
\begingroup\small

\noindent\textbf{Gemma evaluation protocols.} The four-cell study's selected cells report means and standard deviations across training seeds $\{42,123,7,99\}$. All cells use checkpoint step~30 and paired 200-item BFCL v4 \texttt{multi\_turn} evaluation with \texttt{nt=1} deterministic decoding (\S\ref{sec:caseA-predictor}). The category-specific training study reports recovery-run peaks averaged over five evaluation seeds ($\le\!0.015$ std); the selected no-think Gemma-4-26B-A4B checkpoint records \missfunc{} $0.110\!\to\!0.44$ ($+33$pt; supplementary results below).

\paragraph{Four-cell diagnostic sampling.} Each scenario is one BFCL test item, sampled before training with 8 actions and 4 continuations per action. Table~\ref{tab:diagnostic-verification} reports mean corrected action variance $\widehat{\bar V}_{\mathrm{act}}$ estimated from pre-training samples using the estimator in App.~\ref{app:theory}. The mean includes every clean-prefix scenario, with zero and negative estimates retained. The \missparam{} recovery cell has 54 rather than 64 scenarios because 10 had varying model-generated prefixes; the other three cells each retain 64.

\paragraph{Gemma training setup and supplementary results.}\label{sec:caseA-nothink}
We trained one no-think \gemmamodel{}~\cite{gemmateam2026gemma4} model with category-specific turn assignments: \missfunc{} rows train recovery, while \missparam{} rows train decision. The recovery input comprises the decision context $+$ a sampled-but-frozen refusal $+$ the canonical bridge. Only the sampled recovery action receives gradient; its score checks the held call under the same lenient name$+$argument comparator. The in-domain recovery consequence score rises $0.26\!\to\!0.49\!\to\!0.70$ at steps $0/5/10$. Table~\ref{tab:gemma-nothink} reports BFCL results for the category-routed checkpoint.

The decision-turn pilot retained Nemotron's task, data, agent scaffold, and multiplicative reward (Eq.~\eqref{eq:m2-mult}; App.~\ref{sec:caseA-details}), with thinking disabled rather than chain-of-thought. For \missfunc{}, in-domain home-buying validation remains near $0.21$ (start $0.215$), and BFCL accuracy moves $0.110\!\to\!0.135$. Under $z_{\mathrm{tool}}=\mathrm{no\_write}(a_{\mathrm{dec}})\times\mathrm{consequence}(y_{\mathrm{rec}})$, the decision already passes the no-write gate at a high rate ($\mathrm{no\_write}$ near $0.97$), leaving little headroom in this gate, while $\mathrm{consequence}(y_{\mathrm{rec}})$ varies with recovery success. The sampled-group example in \S\ref{sec:noise-example} shows this variation among clean refusals. Missing-argument examples instead expose premature state-changing calls at the decision turn.

\begin{table}[!ht]
\centering
\small
\caption{Supplementary Gemma results: BFCL v4 target-cell accuracy for the no-think Gemma-4-\gemmasuffix{} study.}
\label{tab:gemma-nothink}
\setlength{\tabcolsep}{10pt}
\renewcommand{\arraystretch}{1.15}
\begin{tabular}{l S[table-format=1.3] S[table-format=1.3]}
\toprule
\textbf{Training configuration} & {Missing function} & {Missing argument} \\
\midrule
Starting checkpoint & 0.110 & 0.446 \\
Decision-turn trained & 0.135 & 0.485 \\
Category-routed, selected step~30 & \textbf{0.439} & 0.478 \\
\bottomrule
\end{tabular}
\floatnote{Starting and category-routed rows are means over 5 evaluation seeds; the decision-turn-trained row is single-eval. The category-routed checkpoint is selected by official BFCL v4 weighted full-suite Overall. Its target-cell mean$\pm$std values are $0.439\!\pm\!0.012$ and $0.478\!\pm\!0.006$.}
\end{table}

Across three independent recovery-turn training runs, peak BFCL missing-function accuracies are $0.359\!\pm\!0.006$, $0.351\!\pm\!0.011$, and $0.439\!\pm\!0.012$, compared with $0.110\!\pm\!0.013$ at the start (mean$\pm$std over 5 evaluation seeds).

\FloatBarrier

\par\smallskip\noindent\textbf{Nemotron missing-function application.} The paired gain is $+4.4$pt under repeated deterministic evaluation ($n{=}3$; Nemotron start $0.408\!\pm\!0.013\to0.452\!\pm\!0.021$) and $+6.5$pt under same-day serving ($0.430\to0.495$). Training on an independently designed synthetic task yields a deterministic $+5.0$pt transfer with a $128$k context window. Nemotron-Super-120B checkpoint trajectories use single evaluations; paired serving checks support the reported gains.

\par\smallskip\noindent\textbf{Memory applications.} The selected xLAM checkpoint lies in a cluster of about $6$--$7$ similarly performing checkpoints; the Gemma application evaluates five data-mixture runs (App.~\ref{sec:caseB}).

\paragraph{Nemotron serving variability.} In concurrent vLLM evaluations of Nemotron, temperature $0$ runs vary by about $3$--$5$pt because of batch nondeterminism, and cross-month reruns can shift by about $2$--$3$pt. Evaluation therefore uses the full $128$k context with serial batch-size-$1$ decoding or an independent multi-serve check; shorter windows silently fail long thinking rollouts. The independently trained Nemotron checkpoint produced the same $+5.0$pt difference under serial decoding and a $16$-way concurrent re-serve.

\paragraph{Training units and cost.} Nemotron trains the decision occurrence with the multiplicative label $z_{\mathrm{tool}}=\mathrm{no\_write}(a_{\mathrm{dec}})\times\mathrm{consequence}(y_{\mathrm{rec}})$. In the Nemotron comparison, a whole-trajectory optimizer step cost $16$--$30\times$ as much as an occurrence-local RL step. The xLAM application trains one metadata-selected response with the additive reward and keyword proxy in App.~\ref{sec:caseB-m2}; the selected checkpoint is step~240.

\paragraph{Nemotron renamed-domain training.} We adapt NVIDIA Nemotron-Super-120B~\cite{nvidia2026nemotronsuper} with LoRA~\cite{hu2022lora} (rank 64, $\alpha=256$) using GRPO with DAPO dynamic sampling in NeMo-RL (Megatron) and NeMo-Gym. The learning rate is $5\times10^{-6}$, \texttt{lr\_decay\_iters}\,$=30$, and the KL penalty coefficient is $0$; each step samples 64 prompts with batch multiplier 4 and 32 generations per prompt.

\paragraph{Memory model identity and training stack.} The evaluated checkpoint is an internal, unreleased research checkpoint from the xLAM line~\cite{zhang2025xlam}, used only for research and not part of any product or deployment; the public xLAM series is a separate release. It is a Qwen3-235B-A22B~\cite{yang2025qwen3} fine-tune trained in our pipeline on function-calling data produced by APIGen~\cite{zhang2024xlam}, using verl with Fully Sharded Data Parallelism (FSDP) and GRPO with DAPO dynamic sampling.
\endgroup

\section{Applications}\label{app:applications}
\crosscasedetails
\FloatBarrier

\subsection{Nemotron: missing-function application}\label{sec:caseA-details}
\subsubsection{Scenario, selected occurrence, and label}
\caseAworkedexample
\begin{samepage}
\caseAsetupdetails
\end{samepage}
\caseAnemotronrewarddetails
\subsubsection{Results and evidence scope}\label{sec:caseA-diagnosis}
\begin{table}[H]
\centering
\small
\caption{Nemotron missing-function application: BFCL \missfunc{} accuracy under two training sources. Renamed-domain results summarize repeated deterministic evaluations ($n{=}3$).}
\label{tab:caseA}
\setlength{\tabcolsep}{5pt}
\renewcommand{\arraystretch}{1.12}
\begin{tabular}{@{}>{\raggedright\arraybackslash}p{0.39\linewidth}>{\centering\arraybackslash}p{0.20\linewidth}>{\centering\arraybackslash}p{0.20\linewidth}>{\centering\arraybackslash}p{0.11\linewidth}@{}}
\toprule
\textbf{Training source} & \textbf{Start} & \textbf{RL} & \textbf{Change} \\
\midrule
Renamed home-buying domain & $0.408\pm0.013$ & $0.452\pm0.021$ & $+4.4$pt \\
Hospital-ward & $0.425$ & $0.475$ & $+5.0$pt \\
\bottomrule
\end{tabular}
\end{table}

The hospital-ward result measures transfer from an independently designed training task to unmodified BFCL. Evaluation repetition and serving variability are reported in App.~\ref{app:prov}. Here \emph{turn 0} is the initial decision occurrence in an episode; the turn-0-augmented run adds missing-function training rows at this position. A size-matched control did not support turn-0-specific coverage, so the application result does not rely on that explanation.
\FloatBarrier

\subsection{Memory applications: design and results}\label{sec:caseB}\label{sec:caseB-results}

\memorymethoddetails
\FloatBarrier
\subsubsection{Gemma storage-and-retrieval application}\label{sec:caseB-gemma}\label{app:gemma-mem}
The Gemma application uses a retuned storage$+$retrieval mixture. Storage rows reuse the memory-full candidate, additive reward, and keyword proxy above. Retrieval rows select \texttt{archival\_memory\_search} or \texttt{list\_keys}, graded by expected-query fact coverage. Each training unit updates one selected storage or retrieval response.

\subsubsection{Results}\label{sec:caseB-m4}
\memoryxlamtable

The largest gain is in key--value storage for xLAM and vector storage for Gemma (Table~\ref{tab:caseB}).

\begingroup\small
\section{Extended analysis}\label{app:extended}
\subsection{Repeated occurrences and local credit}
\label{sec:recurrent}

\paragraph{Recurrence result.} Broadcasting one terminal advantage across independent occurrences preserves each occurrence's expected gradient but adds an exact covariance penalty from the other outcomes. Under equal variance, this penalty grows linearly with $K$; occurrence-local credit removes it (Lemmas~\ref{lem:unbiased}--\ref{lem:var}). The resulting per-occurrence sample-cost ratio is explicit in Corollary~\ref{cor:snr}. Even the optimal occurrence-indexed constant baseline leaves the penalty unchanged (Proposition~\ref{prop:flatv}).

\textbf{Relation to the applications.} The $K>1$ model below studies repeated occurrences within an episode, with controlled simulations in App.~\ref{sec:rec:sim}. Missing-function and memory applications use $K{=}1$; logged repeat-call training uses single-step examples from flagged trace positions. App.~\ref{app:theory} examines missing-function label--return associations. The field-survey analysis studies reward aggregation across repeated evaluations of one fixed choice (App.~\ref{sec:rec:evidence}).

\subsubsection{Model and estimators for a recurrent critical state}
\label{sec:rec:model}
Fix $K$ occurrence contexts $i=1,\dots,K$ of an abstract critical state $\sigma$, independent of the episode's sampled outcomes. At occurrence $i$, policy $\pi_\theta$ emits $a_i\in\{1,\dots,m\}$; the correct action $c_i$ may vary with $i$, so $\sigma$ alone need not identify the required action. For example, $\sigma$ may mean ``memory is full,'' while $c_i$ identifies the safe entry to drop \emph{this} time. Let local correctness be $r_i=\mathbf 1[a_i=c_i]$, with $q_i=\Pr_\theta(a_i=c_i)$ and $s_i^2=q_i(1-q_i)=\mathrm{Var}(r_i)$; given $\theta$, occurrences are conditionally independent. Only a terminal scalar is observed:
\[
R=\textstyle\sum_{i=1}^K r_i,\qquad\text{so the return-to-go at \emph{every} occurrence is }G_i=R .
\]
\begin{samepage}
Let $g_i=\nabla_\theta\log\pi_\theta(a_i\mid\sigma,i)$ be the occurrence-$i$ score; the phase-plus-index argument represents the local context, so the policy can distinguish occurrences. We have $\mathbb E[g_i]=0$. Compare the two \emph{per-occurrence gradient contributions}
\[
\underbrace{\hat g_i^{\,\mathrm{GRPO}}=A\,g_i,\quad A=R-b}_{\text{one scalar advantage, broadcast to all }K\text{ turns}}
\qquad\text{vs.}\qquad
\underbrace{\hat g_i^{\,\mathrm{flat}}=(r_i-q_i)\,g_i}_{\text{per-occurrence local label}} ,
\]
\end{samepage}
where $b=\mathbb E[R]=\sum_j q_j$ is the idealized population analogue of GRPO's group mean; all baselines are stop-gradient. Write $A^{\mathrm{flat}}_i:=r_i-q_i$ and $A^{\mathrm{GRPO}}_i:=A=R-b$. This is the independent-reward specialization of factored policy-gradient variance reduction~\cite{spooner2021factored}, using the standard likelihood-ratio and baseline identities~\cite{williams1992reinforce,sutton1999policygrad,greensmith2004variance}. The equations use unnormalized advantages and omit finite-group self-inclusion. Population-standard-deviation normalization, proportional to $\sqrt K$ for shared returns under equal variance, scales signal and noise equally and preserves the gradient-SNR ratios. The results concern one occurrence's contribution; covariance between contributions also enters the gradient summed over shared parameters.

\subsubsection{Exact gradient covariance cost of shared credit}
\label{sec:rec:snr}
Decompose the shared advantage into occurrence $i$'s own signal plus a zero-mean contribution
from the other occurrences,
\[
A=R-\mathbb E[R]=(r_i-q_i)+\xi_i,\qquad \xi_i:=\sum_{j\neq i}(r_j-q_j),\quad \mathbb E[\xi_i]=0,\ \ \xi_i\perp(g_i,r_i).
\]
With $q_i{=}\tfrac12$ at every occurrence, local-label credit for $i$ uses $A^{\mathrm{flat}}_i=r_i-\tfrac12\in\{\pm\tfrac12\}$, whereas shared GRPO uses $A=\sum_j(r_j-\tfrac12)$. At $K{=}2$, ($r_i{=}1,\,r_j{=}0$) gives $A=\tfrac12-\tfrac12=0$; at $K{=}3$, missing both others gives $A=\tfrac12-\tfrac12-\tfrac12<0$. The extra contribution has zero mean (Lemma~\ref{lem:unbiased}), but the variance of $\xi_i$ grows with $K$ (Lemma~\ref{lem:var}).

\begin{lemma}[Same expected gradient under the recurrence model]\label{lem:unbiased}
$\ \mathbb E[\hat g_i^{\,\mathrm{GRPO}}]=\mathbb E[\hat g_i^{\,\mathrm{flat}}]=\mathrm{Cov}(g_i,r_i)=:\mu_i$, independent of $K$ when the occurrence policy and reward distribution are held fixed.
\end{lemma}
\begin{proof}
$\mathbb E[g_i A]=\mathbb E[g_i(r_i-q_i)]+\mathbb E[g_i\xi_i]$. By conditional independence
$\xi_i\perp g_i$ and $\mathbb E[g_i]=0$, so $\mathbb E[g_i\xi_i]=\mathbb E[g_i]\,\mathbb E[\xi_i]=0$;
and $\mathbb E[g_i(r_i-q_i)]=\mathbb E[g_i r_i]-q_i\mathbb E[g_i]=\mathrm{Cov}(g_i,r_i)$, which is
exactly $\mathbb E[\hat g_i^{\,\mathrm{flat}}]$.
\end{proof}

\begin{lemma}[Exact covariance penalty from other occurrences]\label{lem:var}
Let $F_i=\mathbb E[g_i g_i^\top]$ and $v_{-i}=\sum_{j\ne i}s_j^2$. Then
\begin{equation}
\operatorname{Cov}\!\big(\hat g_i^{\,\mathrm{GRPO}}\big)
=\operatorname{Cov}\!\big(\hat g_i^{\,\mathrm{flat}}\big)+v_{-i}F_i .
\label{eq:rec:covariance}
\end{equation}
The added covariance is positive semidefinite. For the scalar multipliers,
\[
\mathrm{Var}\big(A^{\mathrm{flat}}_i\big)=s_i^2,\qquad
\mathrm{Var}\big(A^{\mathrm{GRPO}}_i\big)=\textstyle\sum_{j=1}^K s_j^2
\quad\big(=K s^2\text{ under }s_j\equiv s\big).
\]
\end{lemma}
\begin{proof}
Write $\hat g_i^{\,\mathrm{GRPO}}=\hat g_i^{\,\mathrm{flat}}+\xi_i g_i$. Independence and $\mathbb E[\xi_i]=0$ make the cross covariance zero, while $\operatorname{Cov}(\xi_i g_i)=\mathbb E[\xi_i^2]\mathbb E[g_i g_i^\top]=v_{-i}F_i$. The scalar identity follows by adding independent reward variances.
\end{proof}

\begin{corollary}[Per-occurrence SNR and fixed-precision sample cost]\label{cor:snr}
Use trace covariance for vector variance. Set $v_i=\operatorname{Var}(\hat g_i^{\,\mathrm{flat}})$ and $f_i=\mathbb E\|g_i\|^2$. For $v_i>0$, $\mu_i\ne0$, and independent sampled episodes, the gradient SNR $\|\mu_i\|/\sqrt{\operatorname{Var}(\hat g_i)}$ and sample counts for equal mean-squared error of the gradient sample mean satisfy
\[
\frac{\mathrm{SNR}^{\mathrm{GRPO}}_i}{\mathrm{SNR}^{\mathrm{flat}}_i}
=\sqrt{\frac{v_i}{v_i+v_{-i}f_i}},\qquad
\frac{n_{\mathrm{GRPO}}}{n_{\mathrm{flat}}}=1+\frac{v_{-i}f_i}{v_i}.
\]
Under equal positive reward variance, $v_{-i}=(K-1)s^2$. Holding the occurrence policy fixed gives gradient SNR $\Theta(K^{-1/2})$ and sample-cost ratio $\Theta(K)$. The scalar advantage variances have the exact ratio $K$, hence the advantage-multiplier SNR ratio is exactly $1/\sqrt K$. Equation~\eqref{eq:rec:covariance} also covers $v_i=0$, where a relative sample-cost ratio is undefined.
\end{corollary}

For a trainable occurrence (App.~\ref{app:theory}), the extra estimation cost depends on sibling variance $v_{-i}$ and score magnitude $f_i$, not occurrence count alone (Lemmas~\ref{lem:unbiased} and~\ref{lem:var}). Deterministic sibling outcomes add no penalty. Local credit removes this return noise $\xi_i$; a phase-only baseline retains it (App.~\ref{sec:rec:future}). App.~\ref{sec:stopping} separately analyzes changes to the surrounding policy after the update.

\subsubsection{The residual penalty after optimizing the baseline}
\label{sec:rec:flatv}
\begin{proposition}[Optimal constant baselines leave the covariance penalty]\label{prop:flatv}
At a fixed model context $(\sigma,i)$, let $b_i$ be any constant baseline, including a phase-only $b(\sigma)$, and set $\beta_i=b_i-\sum_{j\ne i}q_j$. Then
\[
\operatorname{Var}\!\big((R-b_i)g_i\big)
=\operatorname{Var}\!\big((r_i-\beta_i)g_i\big)+v_{-i}f_i,
\]
so even optimizing the baseline leaves the exact penalty:
\[
\min_{b_i}\operatorname{Var}\!\big((R-b_i)g_i\big)
=\min_{\beta_i}\operatorname{Var}\!\big((r_i-\beta_i)g_i\big)+v_{-i}f_i.
\]
For $f_i>0$, the local minimizer is $\beta_i^*=\mathbb E[r_i\|g_i\|^2]/f_i$.
\end{proposition}
\begin{proof}
The decomposition $(R-b_i)g_i=(r_i-\beta_i)g_i+\xi_i g_i$ gives the same zero cross covariance as Lemma~\ref{lem:var}. The penalty does not depend on $b_i$. Differentiating $\mathbb E[(r_i-\beta_i)^2\|g_i\|^2]$ yields the minimizer, since the estimator mean is baseline-invariant.
\end{proof}

\paragraph{Phase-only baselines.}\label{sec:rec:future}
For several phase types, a phase-specific baseline $b(\sigma)$ centers each phase separately. Within a recurring phase it is still constant across occurrences, so Proposition~\ref{prop:flatv} leaves the sibling-noise penalty intact.

The constant-baseline result complements methods that use additional information or control variates: future-conditional baselines~\cite{mesnard2021counterfactual}, hindsight credit~\cite{harutyunyan2019hca}, and action-dependent baselines~\cite{tucker2018mirage}. Potential shaping~\cite{ng1999policy}, TD($\lambda$)/GAE~\cite{sutton1988td,schulman2016gae}, and learned return decomposition~\cite{arjona2019rudder,ren2022rrd} provide other credit-assignment mechanisms.

\subsubsection{Controlled simulation: local versus shared credit}
\label{sec:rec:sim}
Figure~\ref{fig:recurrence-toy}a--b and Table~\ref{tab:toy} hold representation fixed and vary credit; panel (c) compares value representations. Per-turn advantage variance is $K\,s^2$ (ratio $=K$, Lemma~\ref{lem:var}). Shared-return GRPO reaches high accuracy with a $10\times$ budget. Measured episodes-to-threshold grow $240\!\to\!1312$ over $K{=}1\!\to\!32$. The $\sigma$-keyed value is flat ($\mathrm{Var}[V]=0$, $R^2=0.00$), versus the acyclic reference ($\mathrm{Var}[V]=1.31$, $R^2=0.54$). Even with homogeneous correct actions, local-label convergence is $K$-constant while shared-return GRPO slows ($14\!\to\!56$ update-steps over $K{=}1\!\to\!16$); Lemma~\ref{lem:var} gives the variance component of this credit-source difference.

\paragraph{State representation.} Figure~\ref{fig:recurrence-toy}c plots the expected remaining local-label sum, $V(s_i)=\mathbb E[\sum_{j=i}^K r_j]$, rather than the value of the terminal-only payment $R$. Distinct occurrence states retain this progress information ($R^2=0.54$); pooling them into a phase-only value gives $R^2=0$. This representation comparison is separate from the shared-credit covariance penalty, which holds with occurrence-distinguishable policies.

\begin{figure}[!htbp]
\centering
\begin{tikzpicture}
\begin{groupplot}[group style={group size=3 by 1, horizontal sep=1.15cm}, cbstyle, width=0.31\linewidth, height=\dimexpr 4.4cm-\baselineskip\relax]
\nextgroupplot[title={\shortstack{(a) shared-return\\variance: $\times K$}}, xlabel={occurrences $K$}, ylabel={per-turn advantage variance}, xmode=log, log basis x=2, xtick={1,2,4,8,16,32}, xticklabels={1,2,4,8,16,32}, ymin=0, ymax=8.4, legend style={font=\scriptsize, at={(0.03,0.97)}, anchor=north west, inner sep=1pt}, legend entries={{$0.25K$},{shared-return GRPO},{local-label credit}}]
\pgfplotsset{legend to name=recurrence-variance-legend}
\addplot[densely dashed,gray] coordinates {(1,0.25)(2,0.5)(4,1)(8,2)(16,4)(32,8)};
\addplot[tolred,mark=*] coordinates {(1,0.25)(2,0.5)(4,1)(8,2)(16,4)(32,7.97)};
\addplot[tolblue,mark=square*] coordinates {(1,0.25)(2,0.25)(4,0.25)(8,0.25)(16,0.25)(32,0.25)};
\nextgroupplot[title={\shortstack{(b) fixed-budget\\per-occurrence accuracy}}, xlabel={occurrences $K$}, ylabel={per-occurrence accuracy}, xmode=log, log basis x=2, xtick={1,2,4,8,16,32}, xticklabels={1,2,4,8,16,32}, ymin=0.58, ymax=1.0, legend pos=south west, legend style={font=\scriptsize}, legend entries={{shared-return GRPO},{local-label credit}}]
\pgfplotsset{legend to name=recurrence-accuracy-legend}
\addplot[tolred,mark=*] coordinates {(1,0.982)(2,0.975)(4,0.956)(8,0.901)(16,0.787)(32,0.627)};
\addplot[tolblue,mark=square*] coordinates {(1,0.982)(2,0.983)(4,0.984)(8,0.984)(16,0.984)(32,0.984)};
\nextgroupplot[title={\shortstack{(c) remaining-label value $V$:\\distinct vs.\ pooled states}}, xlabel={occurrence index $i$}, ylabel={remaining-label value $V$}, xtick={1,2,3,4,5,6,7,8}, ymin=0, ymax=4.3, legend style={font=\scriptsize, at={(0.97,0.97)}, anchor=north east, inner sep=1pt}, legend entries={{distinct states},{shared state ($\sigma$-keyed)}}]
\pgfplotsset{legend to name=recurrence-value-legend}
\addplot[tolgreen,mark=triangle*] coordinates {(1,4)(2,3.5)(3,3)(4,2.5)(5,2)(6,1.5)(7,1)(8,0.5)};
\addplot[tolpurple,mark=diamond*] coordinates {(1,2.25)(2,2.25)(3,2.25)(4,2.25)(5,2.25)(6,2.25)(7,2.25)(8,2.25)};
\end{groupplot}
\node[anchor=north] at (group c1r1.south |- group c1r1.outer south) {\pgfplotslegendfromname{recurrence-variance-legend}};
\node[anchor=north] at (group c2r1.south |- group c2r1.outer south) {\pgfplotslegendfromname{recurrence-accuracy-legend}};
\node[anchor=north] at (group c3r1.south |- group c3r1.outer south) {\pgfplotslegendfromname{recurrence-value-legend}};
\end{tikzpicture}
\caption{Recurrence simulations: shared-return versus local-label credit in variance and fixed-budget accuracy, alongside a remaining-label value comparison.}
\label{fig:recurrence-toy}
\floatnote{Toy model with a recurrent critical state (heterogeneous occurrences, terminal reward $R=\sum_i r_i$; $m{=}5$, group size $16$, $30$ seeds). \textbf{(a)} The theoretical shared-return variance follows the dashed $0.25K$ line; sampled estimates track it while local-label variance stays constant (Lemma~\ref{lem:var}). \textbf{(b)} With a 640-episode budget, shared-return accuracy falls with $K$ while local-label accuracy stays flat across $K$. \textbf{(c)} Expected remaining local-label values decrease across distinct occurrence states; their phase-pooled value is constant.}
\end{figure}
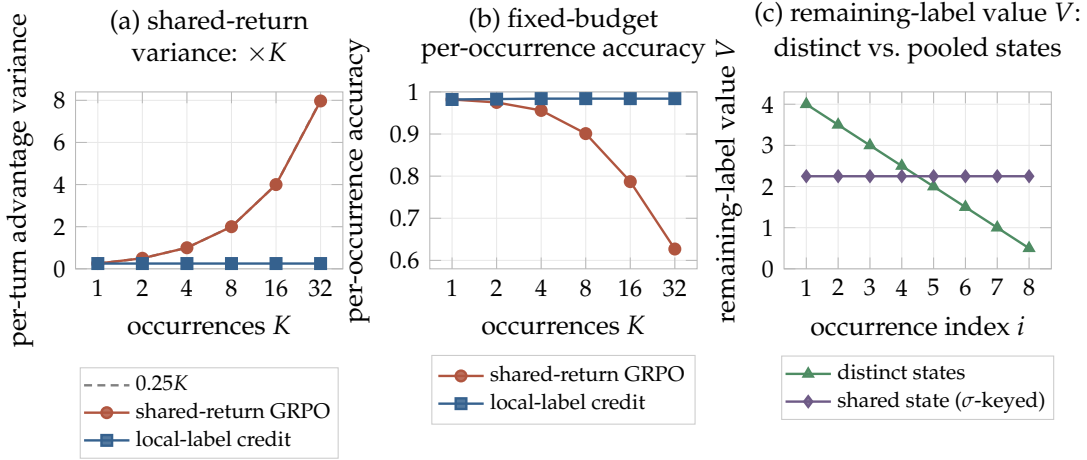

\begin{table}[!htbp]\centering\small
\caption{Controlled recurrence simulation by credit source and sample budget.}
\label{tab:toy}
\renewcommand{\arraystretch}{1.15}
\setlength{\tabcolsep}{5pt}
\begin{tabular}{@{}ll cccccc@{}}
\toprule
Metric & Credit source and budget & \multicolumn{6}{c}{critical-state occurrences $K$}\\
\cmidrule(lr){3-8}
 & & 1 & 2 & 4 & 8 & 16 & 32\\
\midrule
Per-turn advantage variance & \shortstack[l]{Shared-return\\GRPO} & 0.25 & 0.50 & 1.00 & 2.00 & 4.00 & 7.97\\
 & \shortstack[l]{Local-label\\credit} & 0.25 & 0.25 & 0.25 & 0.25 & 0.25 & 0.25\\
\addlinespace
\shortstack[l]{Per-occurrence accuracy\\640-episode reference budget} & Shared-return GRPO & 0.982 & 0.975 & 0.956 & 0.901 & 0.787 & 0.627\\
 & Local-label credit & 0.982 & 0.983 & 0.984 & 0.984 & 0.984 & 0.984\\
 & Shared-return GRPO ($10\times$) & 0.999 & 0.999 & 0.999 & 0.999 & 0.998 & 0.997\\
\addlinespace
Episodes to $0.9$ accuracy & Shared-return GRPO & 240 & 336 & 448 & 648 & 920 & 1312\\
 & Local-label credit & 240 & 232 & 240 & 240 & 240 & 240\\
\bottomrule
\end{tabular}
\floatnote{Heterogeneous occurrences with terminal reward $R=\sum_i r_i$, $m{=}5$ actions, group size $16$, and $30$ seeds. Shared-return GRPO broadcasts one scalar advantage to all $K$ turns; local-label credit uses each occurrence's own label. Advantage-variance rows use a fixed Bernoulli policy ($q{=}\tfrac12$, so $s^2{=}0.25$); accuracy and episode rows are trained outcomes. The reference budget is ``640-ep''; ``$10\times$'' uses a $6400$-ep budget. The factor $K$ is exact analytically (Lemma~\ref{lem:var}); the $K{=}32$ entry ($7.97$ vs.\ $8.0$) is finite-sample.}
\end{table}
\FloatBarrier

\paragraph{Conditions for the variance penalty.}
The homogeneous control confirms that different correct actions are unnecessary for the penalty (App.~\ref{sec:rec:sim}). Occurrence-local credit requires an observable label $r_i$ within the bounded segment, possibly through reward-only continuation (occurrence-local RL $+$ task-specific reward). This label supplies the information that the shared scalar omits.

\subsubsection{Field-survey diagnostic: terminal versus episode return}
\label{sec:rec:evidence}
The base no-think Gemma-4-26B-A4B policy generates complete episodes in a synthetic field-survey task. The agent selects a loadout before the terrain majority is sampled, then keeps it through $K$ field turns that evaluate its consequences. Correctness denotes whether the loadout matches that terrain majority. Terminal $R$ omits a direct penalty for a mismatched loadout; full-episode return-to-go $G_0$ sums per-turn rewards, including loadout-dependent scan outcomes. For either reward quantity, SNR is the absolute correct--wrong mean difference divided by that quantity's standard deviation.

\noindent\textbf{Terminal saturation.} On $29{,}982$ complete field-survey trajectories, terminal $R$ conditioned on loadout correctness is $\mathbb E[R\mid\text{correct}]=0.930$ vs.\ $\mathbb E[R\mid\text{wrong}]=0.951$ ($|\Delta R|=0.021$, $\mathrm{SNR}=0.16$). Across field-turn-count bins $K$, $|\Delta R|$ shrinks from $0.016$ at $K{=}1$ to about $10^{-4}$ for $K{=}7$--$10$ and $K{=}11$--$20$, with both classes having $R$ near $0.999$. Thus terminal discriminability erodes as $K$ grows.

\noindent\textbf{Full-episode discrimination.} The full-episode return-to-go remains discriminative: $\mathbb E[G_0\mid\text{correct}]=3.26$ vs.\ $1.90$ ($\mathrm{SNR}=0.60$, about $3.7$ times that of terminal $R$). The episode return includes intermediate scan rewards, whereas the additive-independent credit comparison in Lemma~\ref{lem:unbiased} and Proposition~\ref{prop:flatv} assigns $G_i=R$.

{\samepage
\noindent\textbf{Mixed subtask outcomes.} $99.9\%$ of trajectories with
$\ge2$ field turns contain both successes and failures in at least one subtask: chore completion or scan matching.
\par}

\subsection{Drift outside the selected occurrence}
\label{sec:stopping}

The exact frozen surrogate evaluates an updated action distribution at the selected call while holding the surrounding policy fixed. A full rollout with the updated model can also change surrounding calls. This \emph{outside-occurrence drift} controls the mismatch between the exact surrogate and the full-rollout objective. The argument uses the standard KL chain rule and Pinsker inequality, with the surrogate comparison familiar from trust-region analysis~\cite{schulman2015trpo}.

\subsubsection{The exact frozen surrogate is a hybrid-policy objective}
Let $P_\theta$ be the distribution of finite-horizon trajectories under $\pi_\theta$, with the same initial distribution and environment for every policy, and let $|R(\tau)|\le R_{\max}$. A fixed rule selects one occurrence $t_\star$ from the history available \emph{before} its action. A trajectory without that occurrence is padded with a policy-independent terminal no-op. The updated action distributions are absolutely continuous with respect to the anchor $\pi_{\theta_0}$ at relevant histories.

Define a hybrid policy that uses $\pi_\theta$ at $t_\star$ and $\pi_{\theta_0}$ everywhere else; write its trajectory law as $P_\theta^{\mathrm{hyb}}$. Let $x$ record the selected pre-action history, including its time and any terminal-padding marker. It has the anchor distribution $d_0$, and its continuation has the anchor return value
\[
Q_{\theta_0}(x,a)=\mathbb E\!\left[R\mid x,a,\ \text{continue under }\pi_{\theta_0}\right].
\]
The exact frozen surrogate is therefore the hybrid policy's expected return:
\begin{equation}
J_{\mathrm{hyb}}(\theta)
=\mathbb E_{P_\theta^{\mathrm{hyb}}}[R]
=\mathbb E_{x\sim d_0,\,a\sim\pi_\theta(\cdot\mid x)}[Q_{\theta_0}(x,a)].
\label{eq:hybrid-objective}
\end{equation}
The full-rollout objective is $J(\theta)=\mathbb E_{P_\theta}[R]$. Both equal $J(\theta_0)$ at the anchor. For a frozen proxy $\widehat Q$, write $\widehat J(\theta)=\mathbb E_{d_0,\pi_\theta}[\widehat Q]$. For each objective, $\Delta$ denotes its change from $\theta_0$. These comparisons hold within the same trajectory task and reward definition.

\subsubsection{Only outside-occurrence drift enters the mismatch bound}
With $h_t$ denoting the pre-action history, define
\[
D_{\mathrm{out}}(\theta)
:=\mathbb E_{\tau\sim P_\theta}\!\left[
\sum_{t\ne t_\star}
\mathrm{KL}\!\left(\pi_\theta(\cdot\mid h_t)\,\|\,\pi_{\theta_0}(\cdot\mid h_t)\right)
\right].
\]

\begin{proposition}[Outside-occurrence drift bound]\label{prop:bias}
Under the preceding construction,
\[
D_{\mathrm{out}}(\theta)
=\mathrm{KL}(P_\theta\,\|\,P_\theta^{\mathrm{hyb}})
\le\mathrm{KL}(P_\theta\,\|\,P_{\theta_0}),
\]
and
\[
\boxed{\quad
\left|\Delta J(\theta)-\Delta J_{\mathrm{hyb}}(\theta)\right|
\le R_{\max}\sqrt{2D_{\mathrm{out}}(\theta)}.
\quad}
\]
\end{proposition}
\begin{proof}
In the trajectory-density ratio $P_\theta/P_\theta^{\mathrm{hyb}}$, the environment factors and selected-action factor cancel. The KL chain rule therefore leaves exactly the conditional KL terms outside $t_\star$. The full KL to the anchor additionally includes the nonnegative selected-action term. Boundedness and Pinsker give
\[
|J(\theta)-J_{\mathrm{hyb}}(\theta)|
\le 2R_{\max}\,\mathrm{TV}(P_\theta,P_\theta^{\mathrm{hyb}})
\le R_{\max}\sqrt{2D_{\mathrm{out}}(\theta)}.
\]
The anchor values agree, so the same bound holds for their gains.
\end{proof}

The square-root order is tight: changing an unselected Bernoulli action from probability $1/2$ to $1/2+u$, with return $\pm R_{\max}$, gives mismatch $2R_{\max}|u|$ and $D_{\mathrm{out}}=2u^2+O(u^4)$.

\paragraph{Behavioral isolation.}
If behavior changes only at the selected occurrence, $D_{\mathrm{out}}=0$ and the exact frozen objective equals full-rollout return, regardless of the selected-action change. Gradient masking alone does not enforce this: shared parameters can change other calls.

\paragraph{When a local update improves the full rollout.}
The score projection in App.~\ref{app:theory} sharpens this comparison near the anchor. Consider a smooth policy on a finite trajectory tree with positive action probabilities. All expectations below use $P_{\theta_0}$. Write $g_t=\nabla_\theta\log\pi_\theta(a_t\mid h_t)|_{\theta_0}$, $s=g_{t_\star}$, and $o=\sum_{t\ne t_\star}g_t$, with zero score for the terminal no-op. Define
\[
\mu_{\mathrm{sel}}=\mathbb E[s(R-\mathbb ER)],\qquad
F_{\mathrm{sel}}=\mathbb E[ss^\top],\qquad
F_{\mathrm{out}}=\mathbb E[oo^\top],\qquad
V_{\mathrm{sel}}^R=\mu_{\mathrm{sel}}^\top F_{\mathrm{sel}}^\dagger\mu_{\mathrm{sel}}.
\]
Here $\mu_{\mathrm{sel}}=\nabla J_{\mathrm{hyb}}(\theta_0)$ and $V_{\mathrm{sel}}^R$ is the return variance explained by the selected score, using the same projection as App.~\ref{app:theory} with terminal return in place of the local label.
\begin{proposition}[Selected-return signal and outside-call drift]\label{prop:local-transfer}
For every parameter direction $u$,
\[
\nabla J(\theta_0)^\top u
\ge \mu_{\mathrm{sel}}^\top u
-\sqrt{\bigl(\operatorname{Var}(R)-V_{\mathrm{sel}}^R\bigr)\,u^\top F_{\mathrm{out}}u}.
\]
A strictly positive right-hand side guarantees improved return for all sufficiently small positive steps along $u$.
\end{proposition}
\begin{proof}
Predictable selection and zero-mean conditional scores give $\mathbb E[so^\top]=0$ and $F_{\mathrm{out}}=\mathbb E\sum_{t\ne t_\star}g_tg_t^\top$. Project centered return onto $s$: the residual $e=R-\mathbb ER-s^\top F_{\mathrm{sel}}^\dagger\mu_{\mathrm{sel}}$ has $\mathbb E e^2=\operatorname{Var}(R)-V_{\mathrm{sel}}^R$ and $\mathbb E[oe]=\mathbb E[oR]$. The policy-gradient identity~\cite{williams1992reinforce} gives $\nabla J(\theta_0)=\mu_{\mathrm{sel}}+\mathbb E[oe]$; Cauchy--Schwarz proves the bound.
\end{proof}
The same matrix controls infinitesimal drift: $D_{\mathrm{out}}(\theta_0+t u)=t^2u^\top F_{\mathrm{out}}u/2+o(t^2)$. Only residual return variation contributes to the first-order return effect of outside-call changes. If $V_{\mathrm{sel}}^R=\operatorname{Var}(R)$, that effect vanishes even when $F_{\mathrm{out}}$ is large. Label-based credit connects to $\mu_{\mathrm{sel}}$ through the calibrated link in Theorem~\ref{thm:rb}.

\paragraph{Proxy error and selected-action drift.}
Suppose $\|\widehat Q-Q_{\theta_0}\|_\infty\le\varepsilon_0$ and define the selected-action KL under \emph{anchor contexts}
\[
D_{\mathrm{sel}}^0(\theta)
:=\mathbb E_{x\sim d_0}\mathrm{KL}(\pi_\theta(\cdot\mid x)\,\|\,\pi_{\theta_0}(\cdot\mid x))
=\mathrm{KL}(P_\theta^{\mathrm{hyb}}\,\|\,P_{\theta_0}).
\]
Then
\begin{equation}
\Delta J(\theta)\ge\Delta\widehat J(\theta)
-R_{\max}\sqrt{2D_{\mathrm{out}}(\theta)}
-2\varepsilon_0\min\!\left\{1,\sqrt{D_{\mathrm{sel}}^0(\theta)/2}\right\}.
\label{eq:proxy-drift-gain}
\end{equation}
Indeed, for $e=\widehat Q-Q_{\theta_0}$, the difference of proxy and exact surrogate gains is $\mathbb E_{d_0}\int(\pi_\theta-\pi_{\theta_0})e\,da$. Its magnitude is at most $2\varepsilon_0\mathbb E_{d_0}\mathrm{TV}(\pi_\theta,\pi_{\theta_0})$; Pinsker and Jensen give \eqref{eq:proxy-drift-gain}. Both error terms vanish at the anchor. The two KL terms use different context distributions: $D_{\mathrm{out}}$ follows current rollouts, whereas $D_{\mathrm{sel}}^0$ follows frozen contexts. Here $\varepsilon_0$ concerns a return-value proxy; the label-to-return scale $\lambda(x)$ in Theorem~\ref{thm:rb} is a separate quantity.

\subsubsection{Optimizing the drift-adjusted gain bound}
For the exact surrogate, parameterize an interval of increasing outside-occurrence drift by $\delta=\sqrt{D_{\mathrm{out}}}$, and let $p(\delta)=\Delta J_{\mathrm{hyb}}$. Proposition~\ref{prop:bias} gives the gain lower bound $G(\delta)=p(\delta)-\kappa\delta$, where $\kappa=\sqrt{2}R_{\max}$.

\begin{corollary}[Optimum of the drift-adjusted lower bound]\label{cor:stop}
If $p$ is continuously differentiable, increasing, and strictly concave on $[0,\delta_{\max}]$, with $p'(0)>\kappa>p'(\delta_{\max})$, then $G$ has a unique interior maximizer satisfying
\[
p'(\delta^\star)=\kappa.
\]
The surrogate continues increasing beyond $\delta^\star$, while its drift-adjusted lower bound decreases.
\end{corollary}
This follows from the sign change of the strictly decreasing derivative $G'=p'-\kappa$.

\paragraph{Measurement and checkpoint selection.}
Estimating $D_{\mathrm{out}}$ requires current-policy full rollouts and summed policy KL outside the selected occurrence. A token-averaged KL on training responses is not that quantity. Evaluate the target multi-turn metric on held-out data for checkpoint selection; where the required KL and proxy-error quantities are available, \eqref{eq:proxy-drift-gain} additionally gives a drift-adjusted gain criterion. Refreshing contexts and continuation values defines a new anchor for the same comparison.

\subsubsection{A stylized stopping example}
The closed-form toy trains scalar $\theta$ on surrogate $\widehat J(\theta)=1-e^{-\theta}$ and sets $J(\theta)=\widehat J(\theta)-B(\theta)$, with $B(\theta)=c\,\theta$ and $\sqrt{\mathrm{KL}}\propto(\theta-\theta_0)$. This stipulated linear penalty illustrates marginal crossing; it is not a fitted outside-occurrence KL envelope. Training sees only $\widehat J$, although the true peak is $\theta^\star=-\ln c$ ($p'=B'$).

At $c=0.15$, $\widehat J$ rises $0\!\to\!0.847$ at step~10$\to\!0.969$ at step~60, whereas true $J$ peaks at $0.565$ at step~10 ($\theta=1.874$, near $\theta^\star=-\ln 0.15=1.897$) and falls to $0.449$ by step~60 (KL $0.040\!\to\!0.135$). With patience-4, evaluation noise $\epsilon_{\mathrm{eval}}\!=\!0.02$, and 400 seeds, true-metric stopping averages step~$9.2$ (KL~$0.036$, reward $0.560$); surrogate stopping averages step~$19.0$ (KL~$0.064$, reward $0.547$). The true peak, $0.565$, is about $26\%$ higher than the step~60 reward, $0.449$.

Raising $c$ from $0.08\!\to\!0.45$ moves $\theta^\star\!:2.53\!\to\!0.80$ (peak step $22\!\to\!2$), following the same marginal-crossing calculation as Cor.~\ref{cor:stop}. Figure~\ref{fig:stopping}a shows the turnover in the fixed-parameter toy.

\subsubsection{Empirical trajectories: transfer turnover and aggregate saturation}\label{sec:caseA-kl}
\paragraph{Transfer turnover (Nemotron-Super-120B, turn-0-augmented run).} The highest observed single-eval BFCL \texttt{miss\_func} transfer point is step~15 ($0.545$, generation-KL about $0.0058$). On the $2450$-sample home-buying set, decision-turn validation instead rises $0.602\!\to\!0.718$ to its sampled maximum at step~35. Later transfer checkpoints remain below the step-15 observation as KL generally rises (Figure~\ref{fig:stopping}b). The observed transfer peak therefore precedes the in-domain maximum, illustrating why the two metrics lead to different checkpoint choices.

\paragraph{Aggregate saturation (no-think Gemma-4-26B-A4B category-routed run).} This run is separate from the four-cell study. In-domain validation continues to improve during \texttt{miss\_func} training (App.~\ref{sec:caseA-nothink}), but official BFCL weighted full-suite \emph{Overall} saturates at $59.2/60.2/59.6/59.4/60.1\%$ over steps $25/30/35/40/45$. Training beyond roughly step~30 brings no aggregate gain, while the untargeted \texttt{base} and \texttt{long\_context} categories fall $0.72\!\to\!0.68$ and $0.61\!\to\!0.53$, respectively, over steps $30\!\to\!45$. The narrower three-category \texttt{multi\_turn} mean therefore turns over (about $0.593$ at step~25, falling to $0.556$ by step~45). Because per-step generation-KL is tiny/noisy at learning rate~$10^{-6}$, the in-domain accuracy climb itself ($0.40\!\to\!0.80$) is the clearer measure of training progress.

\begin{figure}[!htbp]
\centering
\begin{tikzpicture}
\begin{groupplot}[group style={group size=2 by 1, horizontal sep=1.7cm}, cbstyle, width=0.48\linewidth, height=5cm]
\nextgroupplot[title={(a) toy: surrogate vs.\ true objective}, xlabel={training step}, ylabel={objective}, ymin=0, ymax=1.05, legend style={at={(current axis.north west)},yshift=\dimexpr\bigskipamount+\baselineskip\relax,anchor=south west}, legend entries={{surrogate objective $\widehat J$},{true objective $J$}}]
\addplot[tolred,mark=*] coordinates {(0,0)(5,0.74)(10,0.847)(15,0.89)(20,0.91)(30,0.94)(40,0.95)(60,0.969)};
\addplot[tolblue,mark=square*] coordinates {(0,0)(5,0.54)(10,0.565)(15,0.56)(20,0.545)(30,0.52)(40,0.505)(60,0.449)};
\draw[densely dotted,black] (axis cs:10,0) -- (axis cs:10,1.0)
  node[pos=0.68,left=2pt,font=\scriptsize,fill=white,inner sep=1pt]{true peak};
\nextgroupplot[title={(b) Nemotron: validation vs.\ transfer}, xlabel={training step}, ylabel={accuracy}, ymin=0.43, ymax=0.74, legend style={at={(0.5,1.22)},anchor=south, legend columns=1}, legend entries={{in-domain validation},{BFCL missing-function transfer (\texttt{miss\_func})}}]
\addplot[tolred,mark=*] coordinates {(5,0.615)(10,0.66)(15,0.683)(20,0.688)(25,0.708)(30,0.715)(35,0.718)(37,0.708)};
\addplot[tolblue,mark=square*] coordinates {(5,0.445)(10,0.49)(15,0.545)(20,0.48)(25,0.51)(30,0.495)(35,0.47)(37,0.535)};
\draw[densely dotted,black] (axis cs:15,0.43) -- (axis cs:15,0.74) node[pos=0.72,right,font=\scriptsize,align=left]{highest observed\\single-eval point};
\end{groupplot}
\end{tikzpicture}
\caption{Checkpoint-selection trajectories in the stylized toy and Nemotron training run.}
\label{fig:stopping}
\floatnote{\textbf{(a)} In the stylized toy, $\widehat J$ rises monotonically while $J$ turns over. \textbf{(b)} The highest observed transfer point in the Nemotron turn-0-augmented run precedes the in-domain validation maximum. These curves illustrate checkpoint-selection behavior; they do not estimate $D_{\mathrm{out}}$ or calibrate the bound.}
\end{figure}
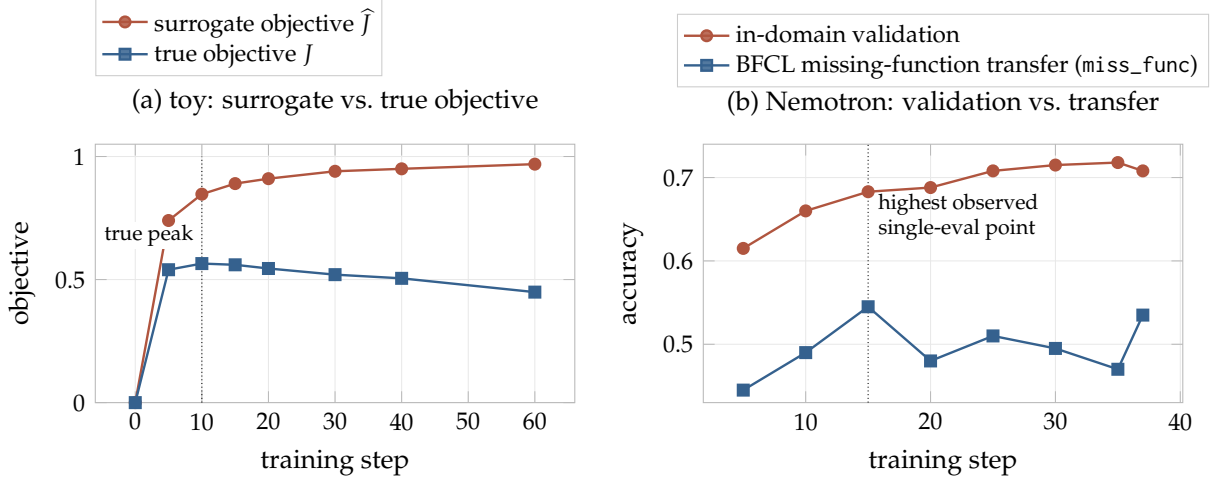

\FloatBarrier

\endgroup

\section{Prompts and canonical strings}
\label{app:prompts}
\begingroup\small
\fvset{fontsize=\scriptsize}

\textbf{Prompt inventory.} The Nemotron-Super-120B and xLAM blocks below are verbatim, including \texttt{\{placeholder\}} tokens.

\subsection{Nemotron tool-use prompts}

\paragraph{Recovery bridge prompt (tool re-availability signal).}
At the decision turn $k{-}1$ the required tool is withheld; at the recovery turn $k$,
training and evaluation use BFCL's canonical tool-reavailability bridge,
\texttt{DEFAULT\_USER\_PROMPT\_FOR\_ADDITIONAL\_FUNCTION\_FC}.

\par\medskip\noindent\begin{minipage}{\linewidth}
\paragraph{Tool-availability judge (turn-0 rows).}
The judge excludes turn-0 requests labeled YES, which indicates that the available
tools can fulfill the request, and retains requests judged to require the withheld tool.
\begin{Verbatim}
_JUDGE_SYSTEM = """You are a careful tool-availability judge for a multi-turn function-calling assistant.

Decide whether the user's request can be reasonably fulfilled by ANY combination of the listed tools. Be strict - answer YES only if a clearly applicable tool exists.

Reply with exactly one word: YES or NO."""

        prompt = [
            {"role": "system", "content": _JUDGE_SYSTEM},
            {"role": "user", "content": (
                f"User request:\n  {user_msg}\n\n"
                f"Available tools:\n{_tools_summary(tools)}\n\n"
                f"Can the user's request be fulfilled? Answer YES or NO."
            )},
        ]
\end{Verbatim}
\end{minipage}\par

\par\medskip\noindent\begin{minipage}{\linewidth}
\paragraph{Pre-bridge assistant messages.}
The \texttt{refusal\_pre\_inject} variants defer or ask about tools at a decision turn
with no ground-truth call. The \texttt{clarification} variants ask for missing parameters.
\begin{Verbatim}[fontsize=\scriptsize]
PRIOR_ASSISTANT_VARIANTS = {
    "refusal_pre_inject": [
        "I don't have a tool that can do that right now - could you enable a function that handles it, or should I keep looking for another way?",
        "None of my currently available tools can do that. Would you like to make a suitable tool available, or shall I try to work around it differently?",
        "I'm not able to do that with the tools I have available. Could you add a tool for it, or should I explore an alternative approach?",
    ],
    "clarification": [
        "I want to make sure I get this right - could you tell me the specific value you'd like me to use?",
        "Could you give me a bit more detail on that so I can fill in the right information?",
        "To do this accurately, which exact value should I use for that?",
    ],
}
\end{Verbatim}
\end{minipage}\par

\subsection{xLAM memory-storage prompts}

\par\medskip\noindent\begin{minipage}{\linewidth}
\paragraph{\texorpdfstring{Shared memory-system instruction (training$\leftrightarrow$evaluation)}{Shared memory-system instruction (training and evaluation)}.}
\begin{Verbatim}[fontsize=\scriptsize]
BACKEND_INSTRUCTION = (
    "You have access to an advanced memory system, consisting of two memory "
    "types 'Core Memory' and 'Archival Memory'. Both type of memory is "
    "persistent across multiple conversations with the user, and can be "
    "accessed in a later interactions. You should actively manage your memory "
    "data to keep track of important information, ensure that it is up-to-date "
    "and easy to retrieve to provide personalized responses to the user later."
    "\n\n"
    "The Core memory is limited in size, but always visible to you in context. "
    "The Archival Memory has a much larger capacity, but will be held outside "
    "of your immediate context due to its size."
)

def build_unified_system_message(memory_state: str) -> str:
    return f"{BACKEND_INSTRUCTION}\n\n{memory_state}"
\end{Verbatim}
\end{minipage}\par

\par\medskip\noindent\begin{minipage}{\linewidth}
\paragraph{Shared memory-full error messages (training$=$evaluation).}
The training data and BFCL evaluator use the same error JSON for failed memory-add calls (training$\leftrightarrow$evaluation).

\noindent
\makeatletter
\begin{minipage}[t]{\dimexpr\linewidth/\tw@-\columnsep\relax}
\begin{Verbatim}[fontsize=\scriptsize]
# Training data (memory-full trigger):
    messages.append({
        "role": "tool",
        "content": '{"error": "Core memory is full. Please clear some entries."}',
        "name": "core_memory_add",
        "tool_call_id": f"call_{idx}_1",
        "tool_calls": None
    })
    ...
        "content": '{"error": "Long term memory is full. Please clear some entries."}',
        "name": "archival_memory_add",

\end{Verbatim}
\end{minipage}\hfill
\begin{minipage}[t]{\dimexpr\linewidth/\tw@-\columnsep\relax}
\begin{Verbatim}[fontsize=\scriptsize]

# BFCL eval checker:
        if len(self.core_memory) >= MAX_CORE_MEMORY_SIZE:
            return {"error": "Core memory is full. Please clear some entries."}
        ...
            return {"error": "Long term memory is full. Please clear some entries."}
\end{Verbatim}
\end{minipage}
\makeatother
\end{minipage}\par

\par\medskip\noindent\begin{minipage}{\linewidth}
\paragraph{Keyword-extraction prompt.}
GPT-4o extracts 2--5 required terms offline into \texttt{extra\_info.keywords}.
The prompt asks it to retain meaningful negations; RL scores storage text by deterministic exact matching.
The system prompt and five examples follow.
\begin{Verbatim}
SYSTEM_PROMPT = """You are a keyword extraction assistant. Given a ground truth value for a memory entry, extract 2-5 key factual terms or phrases that MUST be present in any correct version of this entry.
\end{Verbatim}
\begin{Verbatim}

Focus on:
- Proper nouns (names, places)
- Numbers, dates, times
- Key actions or states (e.g., "high cholesterol", "not married")
- Essential descriptors that distinguish this entry
\end{Verbatim}
\begin{Verbatim}[fontsize=\scriptsize]

Rules:
- Return ONLY a JSON array of lowercase strings
- Each keyword should be 1-3 words
- For very short values (1-2 words), return the value itself as the only keyword
- Avoid generic words like "user", "the", "is", "a"
- Include negations when they are meaningful (e.g., "not" in "not married")"""

FEW_SHOT_EXAMPLES = [
    {
        "value": "blue",
        "keywords": '["blue"]',
    },
    {
        "value": "Blood work 2025-11: cholesterol high (user reported). No numeric values provided. Recommend lifestyle changes and follow-up with PCP.",
        "keywords": '["blood work", "2025-11", "cholesterol high", "lifestyle changes", "pcp"]',
    },
    {
        "value": "User likes art, particularly the impressionists and pointillism. However, user says they are not a good artist themselves.",
        "keywords": '["art", "impressionists", "pointillism", "not a good artist"]',
    },
    {
        "value": "Router disconnects every 20 minutes; first reported by user.",
        "keywords": '["router", "disconnects", "20 minutes"]',
    },
    {
        "value": "User describes themselves as imaginative.",
        "keywords": '["imaginative"]',
    },
]
\end{Verbatim}
\end{minipage}

\endgroup
\end{document}